\documentclass{article}

\usepackage{flushend}
\usepackage{hyperref}
\usepackage{url}

\usepackage{amsmath}
\usepackage{amsfonts}
\usepackage{amsthm}

\usepackage{multirow}

\DeclareMathOperator*{\argmin}{arg\,min}

\usepackage{booktabs}
\usepackage{array}

\usepackage{graphicx}

\usepackage[caption=false]{subfig}

\newtheorem{problem}{Problem}

\usepackage{times}
\usepackage{graphicx} 

\usepackage{xspace}

\usepackage{natbib}

\usepackage{algorithm}
\usepackage[noend]{algorithmic}

\usepackage{hyperref}

\usepackage[accepted]{icml2016}

\newtheorem{thmx}{Theorem}

\renewenvironment{proof}{{\bf Proof:}}{\qed}

\newcommand{\patchysan}{{\textsc{Patchy-san}}\xspace}
\newcommand{\nauty}{{\textsc{Nauty}}\xspace}

\icmltitlerunning{Learning Convolutional Neural Networks for Graphs}

\begin{document} 

\twocolumn[
\icmltitle{Learning Convolutional Neural Networks for Graphs}

\icmlauthor{Mathias Niepert}{mathias.niepert@neclab.eu}
\icmlauthor{Mohamed Ahmed}{mohamed.ahmed@neclab.eu}
\icmlauthor{Konstantin Kutzkov}{konstantin.kutzkov@neclab.eu}
\icmladdress{NEC Labs Europe,
            Heidelberg, Germany}

\icmlkeywords{}

\vskip 0.3in
]

\begin{abstract}
Numerous important problems can be framed as learning from graph data. 
We propose a framework for learning convolutional neural networks for arbitrary graphs. These graphs may be undirected, directed, and with both discrete and continuous node and edge attributes. Analogous to image-based convolutional networks that operate on locally connected regions of the input, we present a general approach to extracting locally connected regions from  graphs. Using established benchmark data sets, we demonstrate that the learned feature representations are competitive with state of the art graph kernels and that their computation is highly efficient. 
\end{abstract}

\section{Introduction}

With this paper we aim to bring convolutional neural networks to bear on a large class of graph-based learning problems. We consider the following two problems.
\begin{enumerate}
\item Given a collection of graphs, learn a function that can be used for classification and regression problems on unseen graphs. The nodes of any two graphs are \emph{not} necessarily in correspondence. For instance, each graph of the collection could model a chemical compound and the output could be a function mapping unseen compounds to their level of activity against cancer cells.
\item Given a large graph, learn graph representations that can be used to infer unseen graph properties such as node types and missing edges. 
\end{enumerate}

We propose a framework for learning representations for classes of directed and undirected graphs. The graphs may have nodes and edges with multiple discrete and continuous attributes and may have multiple types of edges. Similar to convolutional neural network for images, we construct locally connected neighborhoods from the input graphs. These neighborhoods are generated efficiently and serve as the receptive fields of a convolutional architecture, allowing the framework to learn effective graph representations. 

\begin{figure}
\centering
\includegraphics[width = 0.46\textwidth]{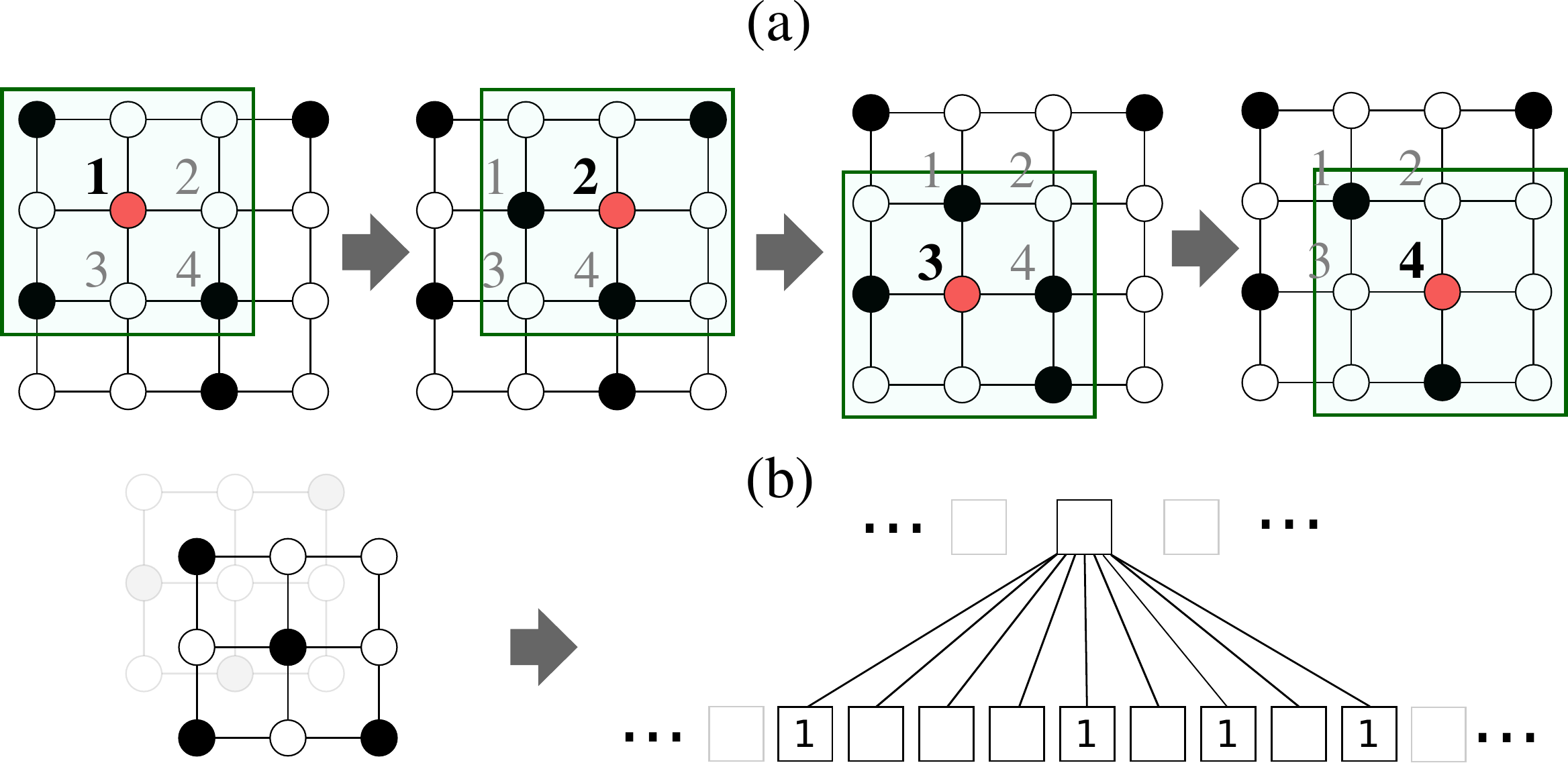}
\caption{\label{fig-grid} A CNN with a receptive field of size $3$x$3$. The field is moved over an image from left to right and top to bottom using a particular stride (here: 1) and zero-padding (here: none) (a). The values read by the receptive fields are transformed into a linear layer and fed to a convolutional architecture (b). The node sequence for which the receptive fields are created and the shapes of the receptive fields are fully determined by the hyper-parameters.}
\end{figure}

The proposed approach builds on concepts from convolutional neural networks (CNNs)~\cite{Kunihiko:1980,atlas:1987,lecun:1998,lecun:2015} for images and extends them to arbitrary graphs. 
Figure~\ref{fig-grid} illustrates the locally connected receptive fields of a CNN for images. An image can be represented as a square grid graph whose nodes represent pixels. Now, a CNN can be seen as traversing a node sequence (nodes $1$-$4$ in Figure~\ref{fig-grid}(a)) and  generating fixed-size neighborhood graphs (the $3$x$3$ grids in Figure~\ref{fig-grid}(b)) for each of the nodes. The neighborhood graphs serve as the receptive fields to read feature values from the pixel nodes. Due to the implicit spatial order of the pixels, the sequence of nodes for which neighborhood graphs are created, from left to right and top to bottom, is uniquely determined. The same holds for NLP problems where each sentence (and its parse-tree) determines a sequence of words. However, for numerous graph collections a problem-specific ordering (spatial, temporal, or otherwise) is missing and the nodes of the  graphs are not in correspondence. In these instances, one has to  solve two problems: (i) Determining the node sequences for which neighborhood graphs are created and (ii) computing a normalization of neighborhood graphs, that is, a unique mapping from a graph representation into a vector space representation. 
The proposed approach, termed \patchysan, addresses these two problems for arbitrary graphs. For each input graph, it first determines nodes (and their order) for which neighborhood graphs are created. For each of these nodes, a neighborhood consisting of exactly $k$ nodes is extracted and normalized, that is, it is uniquely mapped to a space with a fixed linear order. The normalized neighborhood serves as the receptive field for a node under consideration. Finally, feature learning components such as convolutional and dense layers are combined with the normalized neighborhood graphs as the CNN's receptive fields.

Figure~\ref{fig-architecture} illustrates the \patchysan architecture which has several advantages over existing approaches: First, it is highly efficient, naively parallelizable, and applicable to  large graphs. Second, for a number of applications, ranging from computational biology to social network analysis, it is important to visualize learned network motifs~\cite{milo:2002}. \patchysan supports feature visualizations providing insights into the structural properties of graphs. Third, instead of crafting yet another graph kernel, \patchysan learns application dependent features without the need to feature engineering. Our theoretical contributions are the definition of the normalization problem on graphs and its complexity; a method for comparing graph labeling approaches for a collection of graphs; and a result that shows that \patchysan generalizes CNNs on images.  Using standard benchmark data sets, we demonstrate that the learned CNNs for graphs are both efficient and effective compared to state of the art graph kernels. 

\begin{figure}
\centering
\includegraphics[scale = 0.48]{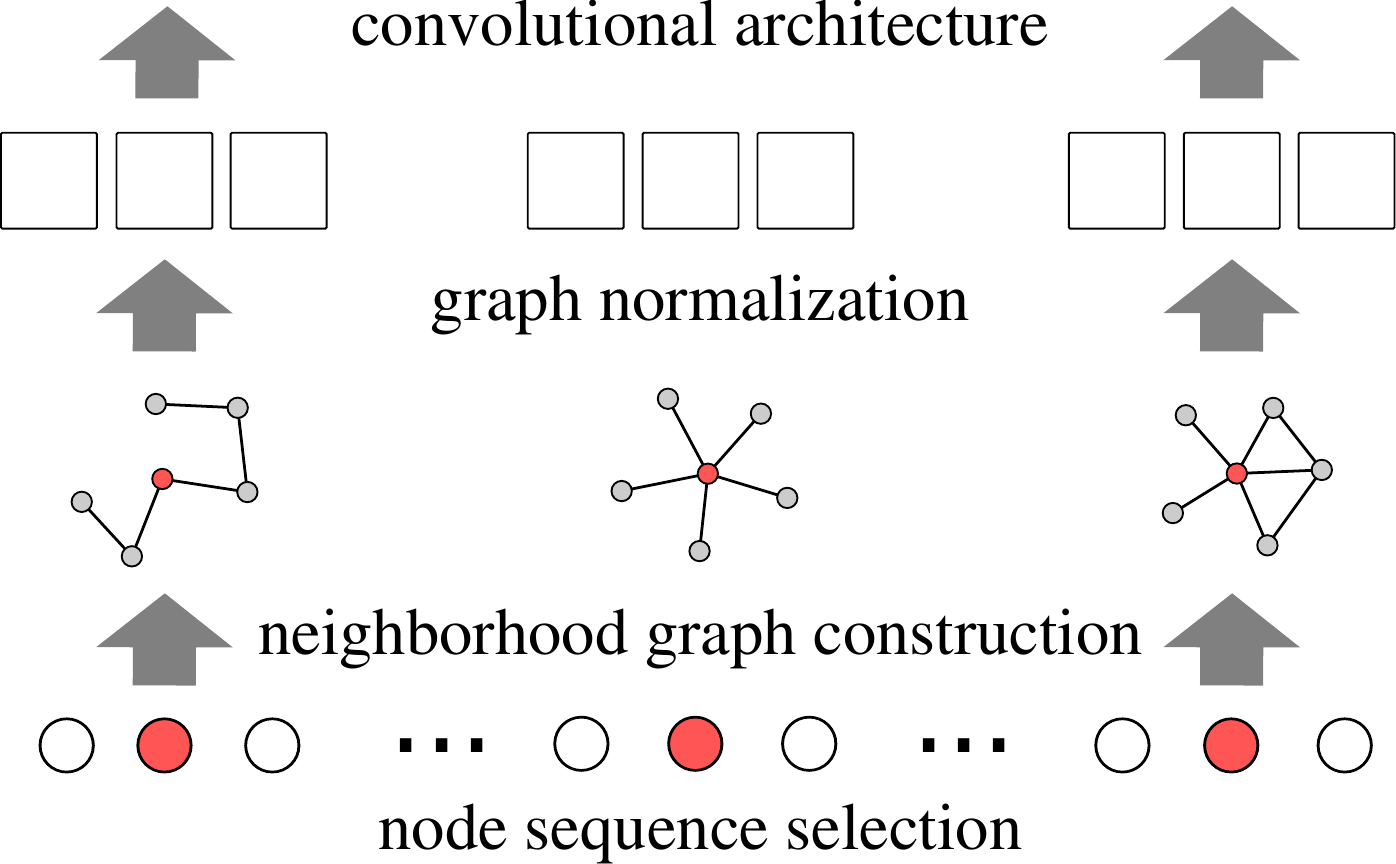}
\caption{\label{fig-architecture} An illustration of the proposed architecture. A node sequence is selected from a graph via a graph labeling procedure. For some nodes in the sequence, a local neighborhood graph is assembled and normalized. The  normalized neighborhoods are used as receptive fields and combined with existing CNN components.}
\end{figure}


\section{Related Work}
\label{sec:related}

Graph kernels allow kernel-based learning approaches such as SVMs to work directly on graphs~\cite{Vishwanathan:2010}. Kernels on graphs were originally defined as similarity functions on the nodes of a single graph~\cite{Kondor:2002}. 
Two representative classes of kernels are the skew spectrum kernel~\cite{Kondor:2008} and kernels based on  graphlets~\cite{Kondor:2009,Shervashidze:2009}. The latter is related to our work, as it builds kernels based on fixed-sized subgraphs. These subgraphs, which are often called motifs or graphlets, reflect functional network properties~\cite{milo:2002,alon:2007}. However, due to the combinatorial complexity of subgraph enumeration, graphlet kernels are restricted to subgraphs with few nodes. 
An effective class of graph kernels are the Weisfeiler-Lehman (WL) kernels~\cite{Shervashidze:2011}. WL kernels, however, only support discrete features and use memory linear in the number of training examples at test time. \patchysan uses WL as one possible labeling procedure to compute receptive fields. 
Deep graph kernels~\cite{Yanardag:2015} and graph invariant kernels~\cite{Orsini:2015} compare graphs based on the existence or count of small substructures such as shortest paths~\cite{Borgwardt:2005}, graphlets, subtrees, and other graph invariants~\cite{haussler:1999,Orsini:2015}.  In contrast, \patchysan learns substructures from graph data and is not limited to a predefined set of motifs. Moreover, while all  graph kernels have a training complexity at least \emph{quadratic} in the number of graphs~\cite{Shervashidze:2011}, which is prohibitive for large-scale problems, \patchysan scales \emph{linearly} with the number of graphs.

Graph neural networks (GNNs)~\cite{Scarselli:2009} are a recurrent neural network architecture defined on graphs. GNNs apply recurrent neural networks for walks on the graph structure, propagating node representations until a fixed point is reached. The resulting node representations are then used as features in classification and regression problems. GNNs support only discrete labels and perform as many backpropagation operations as there are edges and nodes in the graph \emph{per learning iteration}. Gated Graph Sequence Neural Networks modify GNNs to use gated recurrent units and to output sequences~\cite{li:2015}. 

Recent work extended CNNs to topologies that differ from the low-dimensional grid structure~\cite{bruna:2013,henaff:2015}. All of these methods, however, assume one global  graph structure, that is, a correspondence of the vertices across input examples. 
\cite{duvenaud:2015} perform convolutional type operations on graphs, developing a differentiable variant of one specific graph feature.

\section{Background}
We provide a brief introduction to the required background in convolutional networks and graph theory.
\subsection{Convolutional Neural Networks}

CNNs were inspired by earlier work that showed that the visual cortex in animals contains complex arrangements of cells, responsible for detecting light in small local regions of the visual field~\cite{hubel:1968}. CNNs were developed in the $1980$s and have been  applied to image, speech, text, and drug discovery problems~\cite{atlas:1987,LeCun:1989,lecun:1998,lecun:2015,WallachDH:2015}. A predecessor to CNNs was the Neocognitron~\cite{Kunihiko:1980}. 
A typical CNN is composed of convolutional and dense layers. The purpose of the first convolutional layer is the extraction of common patterns found within local regions of the input images. CNNs convolve learned filters over the input image, computing the inner product at every image location
in the image and outputting the result as tensors whose depth is the number of filters.

\subsection{Graphs}

A graph $G$ is a pair $(V, E)$ with $V = \{v_1, ..., v_n\}$ the set of vertices and $E \subseteq V \times V$ the set of edges. Let $n$ be the number of vertices and $m$ the number of edges. 
 Each graph can be represented by an adjacency matrix $\mathbf{A}$ of size $n \times n$, where $\mathbf{A}_{i,j} = 1$ if there is an edge from vertex $v_i$ to vertex 
$v_j$, and $\mathbf{A}_{i,j} = 0$ otherwise. In this case, we say that vertex $v_i$ has \emph{position} $i$ in $\mathbf{A}$. Moreover, if $\mathbf{A}_{i,j} = 1$ we say $v_i$ and $v_j$ are \emph{adjacent}. Node and edge attributes are features that attain one value for each node and edge of a graph. We use the term attribute value instead of label to avoid confusion with the graph-theoretical concept of a labeling. 
A walk is a sequence of nodes in a graph, in which consecutive nodes are connected by an edge. A path is a walk with distinct nodes.
We write $\mathbf{d}(u, v)$ to denote the distance between $u$ and $v$, that is, the length of the shortest path between $u$ and $v$.
$N_1(v)$ is the $1$-neighborhood of a node, that is, all nodes that are adjacent to $v$. 


\textbf{Labeling and Node Partitions.}
\textsc{Patchy-san} utilizes graph labelings to impose an order on nodes. 
A graph labeling $\ell$ is a function $\ell: V \rightarrow S$  from the set of vertices $V$ to an ordered set $S$ such as the real numbers and integers. A graph labeling procedure computes a graph labeling for an input graph. When it is clear from the context, we use \emph{labeling} to refer to both, the graph labeling and the  procedure to compute it.  A ranking (or coloring) is a function $\mathbf{r}:V \rightarrow \{1, ..., |V|\}$. Every labeling induces a ranking $\mathbf{r}$ with $\mathbf{r}(u) < \mathbf{r}(v)$ if and only if $\ell(u)>\ell(v)$. 
If the labeling $\ell$ of graph $G$ is injective, it determines a total order of $G$'s vertices and a unique adjacency matrix $\mathbf{A}^{\ell}(G)$ of $G$ where vertex $v$ has position $\mathbf{r}(v)$ in $\mathbf{A}^{\ell}(G)$. Moreover, every graph labeling induces a partition $\{V_1, ..., V_n\}$ on $V$  with $u, v \in V_i$ if and only if $\ell(u)=\ell(v)$. 

Examples of graph labeling procedures are node degree and other measures of centrality commonly used in the analysis of networks. For instance, the \emph{betweeness centrality} of a vertex $v$ computes the fractions of shortest paths that pass through $v$. 
The Weisfeiler-Lehman algorithm~\cite{weisfeiler:1968,douglas2011weisfeiler} is a procedure for partitioning the vertices of a graph. 
It is also known as color refinement and naive vertex classification. 
Color refinement has attracted considerable interest in the ML community since it can be applied to speed-up inference in graphical models~\cite{Kersting:2009,Kersting:2014} and as a method to compute graph kernels~\cite{Shervashidze:2011}. 
\patchysan applies these labeling procedures, among others (degree, page-rank, eigenvector centrality, etc.), to impose an order on the nodes of graphs, replacing application-dependent orders (temporal, spatial, etc.) where missing.

\textbf{Isomorphism and Canonicalization.}
The computational problem of deciding whether two graphs are isomorphic surfaces in several application domains. 
The graph isomorphism (GI) problem is in NP but not known to be in P or NP-hard. Under several mild restrictions, GI is known to be in P. For instance, GI is in P for graphs of bounded degree~\cite{luks:1982}. A canonicalization of a graph $G$ is a graph $G'$ with a fixed vertex order which is isomorphic to $G$ and which represents its entire isomorphism class. In practice, the graph canonicalization tool \nauty has shown remarkable performance~\cite{McKay:2014}.

\section{Learning CNNs for Arbitrary Graphs}

When CNNs are applied to images, a receptive field (a square grid) is moved over each image with a particular step size. The receptive field reads the pixels' feature values,  for each channel once, and a patch of values is created for each channel. 
Since the pixels of an image have an implicit arrangement -- a spatial order -- the receptive fields are always moved from left to right and top to bottom. Moreover, the spatial order uniquely determines the nodes of each receptive field and the way these nodes are mapped to a vector space representation (see Figure~\ref{fig-grid}(b)). Consequently, the values read from two pixels using two different locations of the receptive field are assigned to the same relative position if and only if the pixels' structural roles (their spatial position within the receptive field) are identical.

To show the connection between CNNs and \patchysan, we frame CNNs on images as identifying a sequence of nodes in the square grid graph representing the image and building a normalized neighborhood graph -- a receptive field -- for each node in the identified sequence.
For graph collections where an application-dependent node order is missing and where the nodes of any two graphs are not yet aligned, we need to determine for each graph (i) the sequences of nodes for which we create neighborhoods, and (ii) a unique mapping from the graph representation to a vector representation such that nodes with similar structural roles in the neighborhood graphs are positioned similarly in the vector representation. 

We address these problems by leveraging graph labeling procedures that assigns nodes from two different graphs to a similar relative position in their respective adjacency matrices if their structural roles within the graphs are similar. Given a collection of graphs, \patchysan (\textsc{Select}-\textsc{Assemble}-\textsc{Normalize}) applies the following steps to each graph: (1) Select a fixed-length sequence of nodes from the graph; (2) assemble a fixed-size neighborhood for each node in the selected sequence; (3) normalize the extracted  neighborhood graph; and (4) learn neighborhood representations with convolutional neural networks from the resulting sequence of patches.  

In the following, we describe methods that address the above-mentioned challenges. 

\begin{algorithm}[t!]
  \small
   \caption{\textsc{SelNodeSeq}: Select Node Sequence}
   \label{alg:sequence}
\begin{algorithmic}[1]
   \STATE {\bfseries input:} graph labeling procedure $\ell$, graph $G=(V, E)$, stride $s$, width $w$, receptive field size $k$
   \STATE $V_{\mathtt{sort}}$ = top $w$ elements of $V$ according to $\ell$
   \STATE $i=1, j=1$
   \WHILE{$j < w$}
   \IF{$i \leq |V_{\mathtt{sort}}|$}
   \STATE $\mathsf{f} = \textsc{ReceptiveField}(V_{\mathtt{sort}}[i])$
   \ELSE
   \STATE $\mathsf{f} = \textsc{ZeroReceptiveField}()$
   \ENDIF
   \STATE apply $\mathsf{f}$ to each input channel
   \STATE $i = i + s$, $j = j + 1$
   \ENDWHILE
\end{algorithmic}
\end{algorithm}
\vspace{-1mm}

\begin{figure*}
\centering
\includegraphics[scale = 0.462]{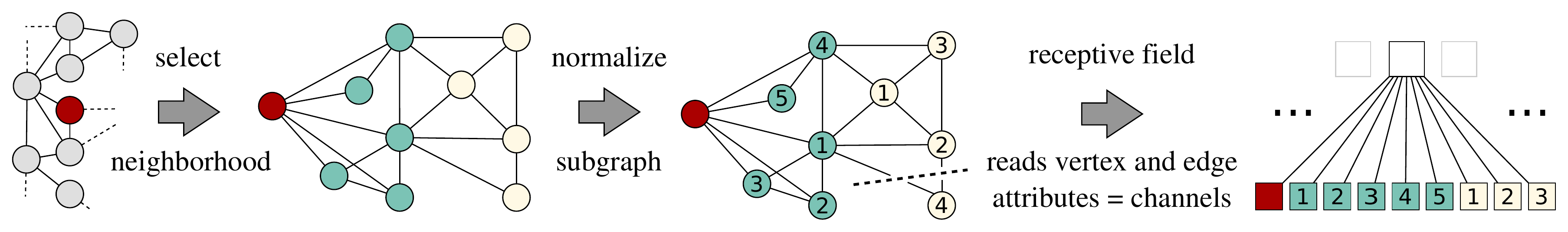}
\vspace{-1.6mm}
\caption{\label{fig-normalize} The normalization is performed for each of the graphs induced on the neighborhood of a root node $v$ (the red node; node colors indicate distance to the root node). A graph labeling is used to rank the nodes and to create the normalized receptive fields, one of size $k$ (here: $k=9$) for node attributes and one of size $k\times k$ for edge attributes. Normalization also includes cropping of excess nodes and padding with dummy nodes. Each vertex (edge) attribute corresponds to an input channel with the respective receptive field.}
\end{figure*}

\subsection{Node Sequence Selection}
Node sequence selection is the process of identifying, for each input graph, a sequence of nodes for which receptive fields are created. Algorithm~\ref{alg:sequence} lists one such procedure. First, the vertices of the input graph are sorted with respect to a given graph labeling. Second, the resulting node sequence is traversed using a given stride $s$ and for each visited node, Algorithm~\ref{alg:example} is executed to construct a receptive field, until exactly $w$ receptive fields have been created. The stride $s$ determines the distance, relative to the selected node sequence, between two consecutive nodes for which a receptive field is created. If the number of nodes is smaller than $w$, the algorithm creates all-zero receptive fields for padding purposes. 

Several alternative methods for vertex sequence selection are possible. For instance, a depth-first traversal of the input graph guided by the values of the graph labeling. We leave these ideas to future work.

\begin{algorithm}[t!]
  \small
   \caption{\label{alg:assembly}\textsc{NeighAssemb}: Neighborhood Assembly}   
\begin{algorithmic}[1]
   \STATE {\bfseries input:} vertex $v$, receptive field size $k$
   \STATE {\bfseries output:} set of neighborhood nodes $N$ for $v$
   \STATE $N  = [v]$
   \STATE $L = [v]$
   \WHILE{$|N| < k$ and $|L|>0$}
   \STATE $L = \bigcup_{v \in L} N_1(v)$
   \STATE $N = N \cup L$
   \ENDWHILE
   \STATE {\bfseries return} the set of vertices $N$
\end{algorithmic}
\end{algorithm}
\vspace{-1mm}

\subsection{Neighborhood Assembly}

For each of the nodes identified in the previous step, a receptive field has to be constructed. Algorithm~\ref{alg:example} first calls Algorithm~\ref{alg:assembly} to assembles a local neighborhood for the input node. The nodes of the neighborhood are the candidates for the receptive field. Algorithm~\ref{alg:assembly} lists the neighborhood assembly steps. Given as inputs a node $v$ and the size of the receptive field $k$, the procedure performs a breadth-first search, exploring vertices with an increasing distance from $v$, and adds these vertices to a set $N$. If the number of collected nodes is smaller than $k$, the $1$-neighborhood of the  vertices most recently added to $N$ are collected, and so on, until at least $k$ vertices are in $N$, or until there are no more neighbors to add. Note that at this time, the size of $N$ is possibly different to $k$.


\subsection{Graph Normalization}

The receptive field for a node is constructed by \emph{normalizing} the neighborhood assembled in the previous step. Illustrated in Figure~\ref{fig-normalize}, the normalization imposes an order on the nodes of the neighborhood graph so as to map from the unordered graph space to a vector space with a linear order. The basic idea is to leverage graph labeling procedures that assigns nodes of two different graphs to a similar relative position in the respective adjacency matrices if and only if their structural roles within the graphs are similar.


To formalize this intuition, we define the optimal graph normalization problem which aims to find a labeling that is optimal relative to a given collection of graphs. 

\begin{problem}[Optimal graph normalization]
Let $\mathcal{G}$ be a collection of unlabeled graphs with $k$ nodes, let $\ell$ be an injective graph labeling procedure, let $\mathbf{d}_{\mathbf{G}}$ be a distance measure on graphs with $k$ nodes, and let $\mathbf{d}_{\mathbf{A}}$ be a distance measure on $k\times k$ matrices. Find $\hat{\ell}$ such that
$$ \label{equation-graph-normalization}
\hat{\ell} = \argmin_{\ell} \mathbb{E}_{\mathcal{G}}\left[\left|\mathbf{d}_{\mathbf{A}}\left(\mathbf{A}^{\ell}(G), \mathbf{A}^{\ell}(G')\right) - \mathbf{d}_{\mathbf{G}}(G, G')\right|\right]. $$
\end{problem}
The problem amounts to finding a graph labeling procedure $\ell$, such that, for any two graphs drawn uniformly at random from $\mathcal{G}$, the expected difference between the distance of the graphs in vector space (with respect to the adjacency matrices based on $\ell$) and the distance of the graphs in graph space is minimized. The optimal graph normalization problem is a generalization of the classical graph canonicalization problem. A canonical labeling algorithm, however, is optimal only for isomorphic graphs and might perform poorly for graphs that are similar but not isomorphic. In contrast, the smaller the expectation of the optimal normalization problem, the better the labeling aligns nodes with similar structural roles. Note that the similarity is determined by $\mathbf{d}_{\mathbf{G}}$.

\begin{algorithm}[t!]
  \small
   \caption{\textsc{ReceptiveField}: Create Receptive Field}
   \label{alg:example}
\begin{algorithmic}[1]
	\STATE {\bfseries input:} vertex $v$, graph labeling $\ell$, receptive field size $k$
   \STATE $N = \textsc{NeighAssemb}(v, k)$
   \STATE $G_{\mathtt{norm}} = \textsc{NormalizeGraph}(N, v, \ell, k)$
   \STATE {\bfseries return} $G_{\mathtt{norm}}$
\end{algorithmic}
\end{algorithm}

We have the following result concerning the complexity of the optimal normalization problem. 
\begin{thmx} \label{thm:normalization_NP}
  Optimal graph normalization is NP-hard.
\end{thmx}
\begin{proof}
By reduction from subgraph isomorphism.
\end{proof}

\patchysan does \emph{not} solve the above optimization problem. Instead, it may compare different graph labeling methods and choose the one that performs best relative to a given collection of graphs.

\begin{algorithm}[t!]
   \small
   \caption{\textsc{NormalizeGraph}: Graph Normalization}
   \label{alg:normalization}
\begin{algorithmic}[1]
   \STATE {\bfseries input:} subset of vertices $U$ from original graph $G$, vertex $v$, graph labeling $\ell$, receptive field size $k$
   \STATE {\bfseries output:}  receptive field for $v$
   \STATE compute ranking $\mathbf{r}$ of $U$ using $\ell$, subject to \\  $\forall u,w \in U: \mathbf{d}(u,v) < \mathbf{d}(w,v) \Rightarrow \mathbf{r}(u)<\mathbf{r}(w)$
   \IF{$|U|> k$}
   \STATE $N = $ top $k$ vertices in $U$ according to $\mathbf{r}$
   \STATE compute ranking $\mathbf{r}$ of $N$ using $\ell$, subject to \\ $\forall u,w \in N: \mathbf{d}(u,v) < \mathbf{d}(w,v) \Rightarrow \mathbf{r}(u)<\mathbf{r}(w)$ 
   \ELSIF{$|V| < k$}
   \STATE $N = U$ and $k-|U|$ dummy nodes
   \ELSE 
   \STATE $N = U$
   \ENDIF
   \STATE construct the subgraph $G[N]$ for the vertices $N$
   \STATE canonicalize $G[N]$, respecting the prior coloring $\mathbf{r}$
   \STATE {\bfseries return}  $G[N]$
\end{algorithmic}
\end{algorithm}

\begin{thmx}~\label{thrm:graph_norm_expectation}
Let $\mathcal{G}$ be a collection of graphs and let $(G_1,G_1'), ..., (G_N,G_N')$ be a sequence of pairs of graphs sampled independently and uniformly at random from $\mathcal{G}$. Let 
$\hat{\theta}_{\ell} := \sum_{i=1}^{N} \mathbf{d}_{\mathbf{A}}\left( \mathbf{A}^{\ell}(G_i), \mathbf{A}^{\ell}(G_i')\right) / N$ and $\theta_{\ell} := \mathbb{E}_{\mathcal{G}}\left[\left| \mathbf{d}_{\mathbf{A}}\left( \mathbf{A}^{\ell}(G), \mathbf{A}^{\ell}(G')\right) - \mathbf{d}_{\mathbf{G}}(G, G')\right|\right]$. If $\mathbf{d}_{\mathbf{A}} \geq \mathbf{d}_{\mathbf{G}}$, then $\mathbb{E}_{\mathcal{G}}[\hat{\theta}_{\ell_1}] < \mathbb{E}_{\mathcal{G}}[\hat{\theta}_{\ell_2}]$ if and only if $\theta_{\ell_1} < \theta_{\ell_2}$.
\end{thmx}

Theorem~\ref{thrm:graph_norm_expectation} enables us to compare different labeling procedures in an unsupervised manner via a comparison of the corresponding estimators.  Under the assumption $\mathbf{d}_{\mathbf{A}} \geq \mathbf{d}_{\mathbf{G}}$, the smaller the estimate $\hat{\theta}_{\ell}$ the smaller the absolute difference. Therefore, we can simply choose the labeling $\ell$ for which $\hat{\theta}_{\ell}$ is minimal. The assumption $\mathbf{d}_{\mathbf{A}} \geq \mathbf{d}_{\mathbf{G}}$ holds, for instance, for the edit distance on graphs and the Hamming distance on adjacency matrices. Finally, note that all of the above results can be extended to directed graphs. 

The graph normalization problem and the application of appropriate graph labeling procedures for the normalization of local graph structures is at the core of the proposed approach. Within the \patchysan framework, we normalize the neighborhood graphs of a vertex $v$. The labeling of the vertices is therefore constrained by the graph distance to  $v$: for any two vertices $u, w$, if $u$ is closer to $v$ than $w$, then $v$ is always ranked higher than $w$. 
This definition ensures that $v$ has always rank $1$, and that the closer a vertex is to $v$ in $G$, the higher it is ranked in the vector space representation.

Since most labeling methods are not injective, it is necessary to break ties between same-label nodes. To do so, we use \nauty~\cite{McKay:2014}. \nauty accepts prior node partitions as input and breaks remaining ties by choosing the lexicographically maximal adjacency matrix. It is known that graph isomorphism is in PTIME for graphs of bounded degree~\cite{luks:1982}. Due to the constant size $k$ of the neighborhood graphs, the algorithm runs in time polynomial in the size of the original graph and, on average, in time linear in $k$~\cite{Babai:1980}. Our experiments verify that computing a canonical labeling of the graph neigborhoods adds a negligible overhead.

Algorithm~\ref{alg:normalization} lists the normalization procedure. If the size of the input set $U$ is larger than $k$, it first applies the ranking based on $\ell$ to select the top $k$ nodes and recomputes a ranking on the smaller set of nodes. If the size of $U$ is smaller than $k$, it adds disconnected dummy nodes. Finally, it induces the subgraph on the vertices $N$ and canonicalizes the graph taking the ranking $\mathbf{r}$ as prior coloring.

We can relate \textsc{Patchy-san} to CNNs for images as follows. 
\begin{thmx}
Given a sequence of pixels taken from an image. Applying \textsc{Patchy-san} with receptive field size $(2m-1)^2$, stride $s$, no zero padding, and $1$-WL normalization to the sequence is identical (up to a fixed permutation of the receptive field) to the first layer of a \textsc{CNN} with receptive field size $2m-1$, stride $s$, and no zero padding.  
\end{thmx}

\begin{proof}
It is possible to show that if an input graph is a square grid, then the $1$-WL normalized receptive field constructed for a vertex is always a square grid graph with a unique vertex order. 
\end{proof}

\subsection{Convolutional Architecture}

\patchysan is able to process both vertex and edge attributes (discrete and continuous). Let $\mathtt{a}_v$ be the number of vertex attributes and let $\mathtt{a}_e$ be the number of edge attributes. For each input graph $G$, it applies normalized receptive fields for vertices and edges which results in one $(w, k, \mathtt{a}_v)$ and one $(w, k, k, \mathtt{a}_e)$ tensor. These can be reshaped to a $(wk, \mathtt{a}_v)$ and a $(wk^2, \mathtt{a}_e)$ tensors. Note that $\mathtt{a}_v$ and $\mathtt{a}_e$ are the number of input channels. We can now apply a $1$-dimensional convolutional layer with stride and receptive field size $k$ to the first and $k^2$ to the second tensor. The rest of the architecture can be chosen arbitrarily. We may use merge layers to combine convolutional layers representing nodes and edges, respectively.

\section{Complexity and Implementation}
\patchysan's algorithm for creating receptive fields is highly efficient and  naively parallelizable because the fields are generated independently. We can show the following asymptotic worst-case result.

\begin{thmx}
Let $N$ be the number of graphs, let $k$ be the receptive field size, $w$ the width, and $O(f(n,m))$ the complexity of computing a given labeling $\ell$ for a graph with $n$ vertices and $m$ edges. \textsc{Patchy-san} has a worst-case complexity of $O(N w (f(n,m)+ n\log(n) + \exp(k)))$ for computing the receptive fields for $N$ graphs.
\end{thmx}
\begin{proof}
Node sequence selection requires the labeling of each input graph and the retrieval of the $k$ highest ranked nodes. For the creation of normalized graph patches, most computational effort is spent applying the labeling procedure $\ell$ to a neighborhood whose size may be larger than $k$. Let $\overline{d}$ be the maximum degree of the input graph $G$, and $U$ the neighborhood returned by Algorithm~\ref{alg:assembly}. We have $|U| \leq (k-2) \overline{d} \leq n$.
The term $\exp(k)$ comes from the worst-case complexity of the graph canonicalization algorithm \nauty on a $k$ node graph~\cite{miyazaki:1997}. 
\end{proof} 

For instance, for the Weisfeiler-Lehman algorithm, which has a complexity of  $O((n+m) \log(n))$~\cite{Berkholz:2013}, and constants $w \ll n$ and $k \ll n$, the complexity of \textsc{Patchy-san} is linear in $N$ and quasi-linear in $m$ and $n$. 



\section{Experiments}
We conduct three types of experiments: a runtime analysis, a qualitative analysis of the learned features, and a comparison to graph kernels on benchmark data sets. 

\begin{figure}
\includegraphics[width=0.47\textwidth]{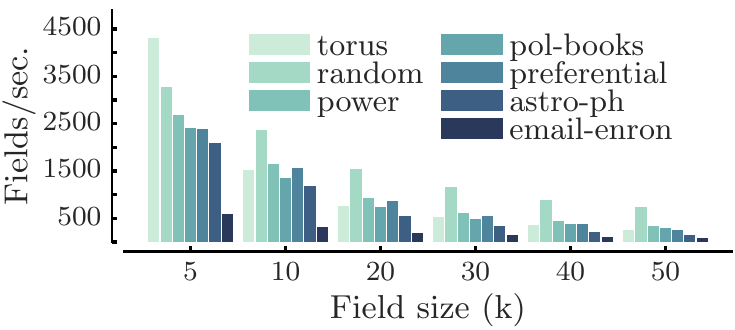}
\caption{\label{fig-runtime} Receptive fields per second rates on different graphs.}
\end{figure}

\begin{figure*}[t!]
\captionsetup[subfigure]{labelformat=empty}
\centering
\begin{tabular}{m{0.6in} m{0.6in} m{0.4in} m{0.4in} m{0.4in} m{0.1in} m{0.6in} m{0.6in} m{0.4in} m{0.4in} m{0.4in}}
\subfloat[]{\includegraphics[width = 0.8in]{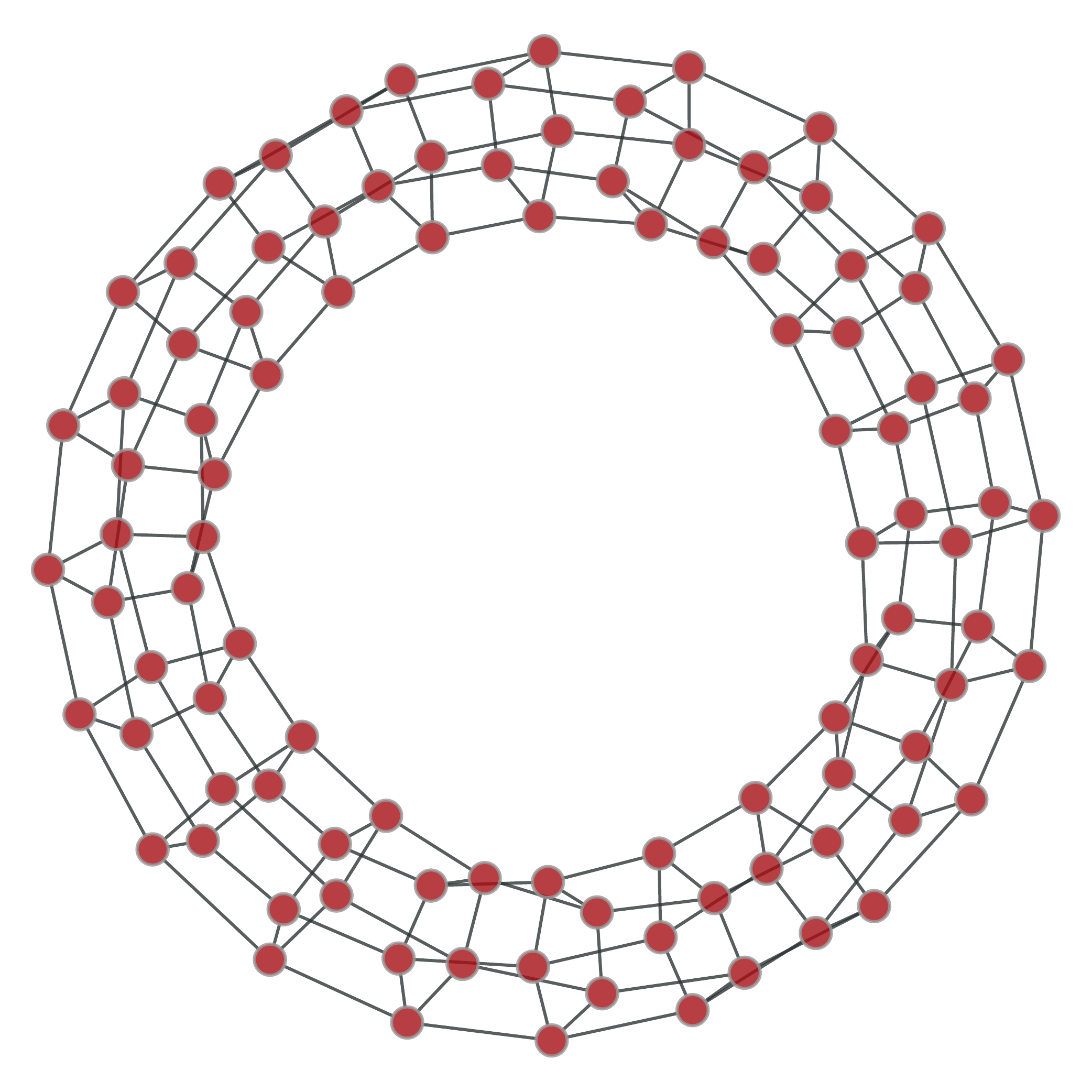}} &
\subfloat[]{\includegraphics[width = 0.7in]{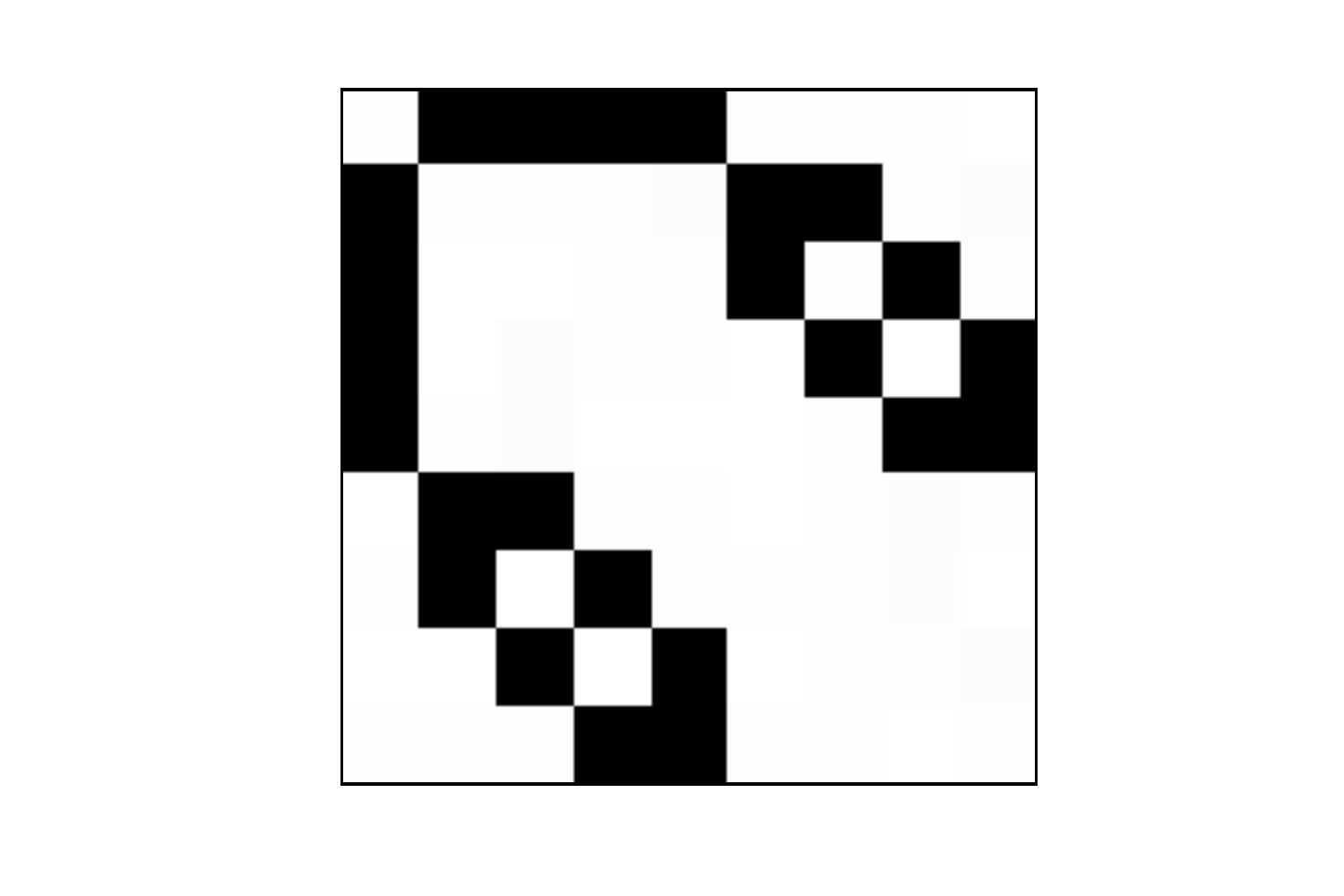}} &
\subfloat[]{\includegraphics[width = 0.4in]{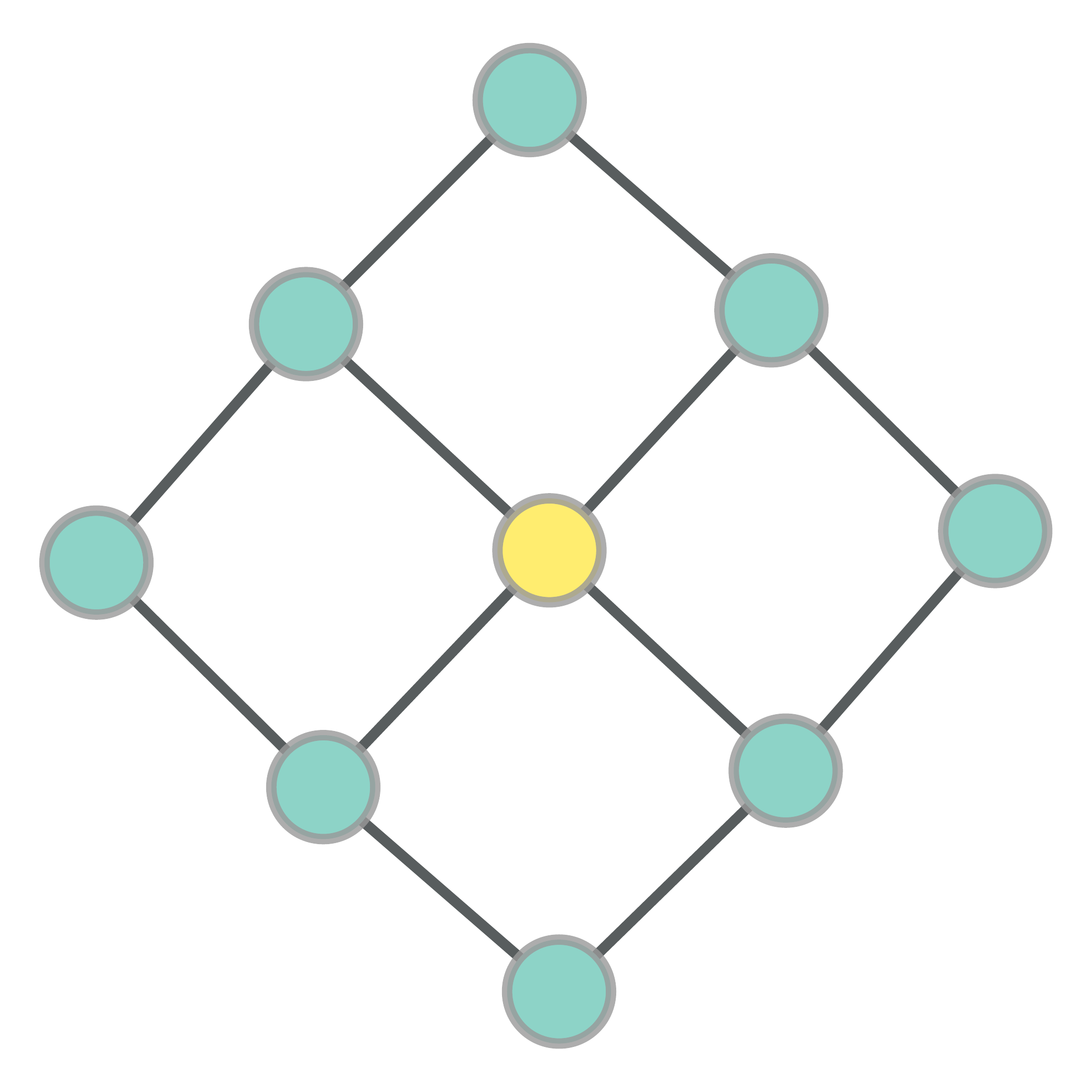}} &
\subfloat[]{\includegraphics[width = 0.4in]{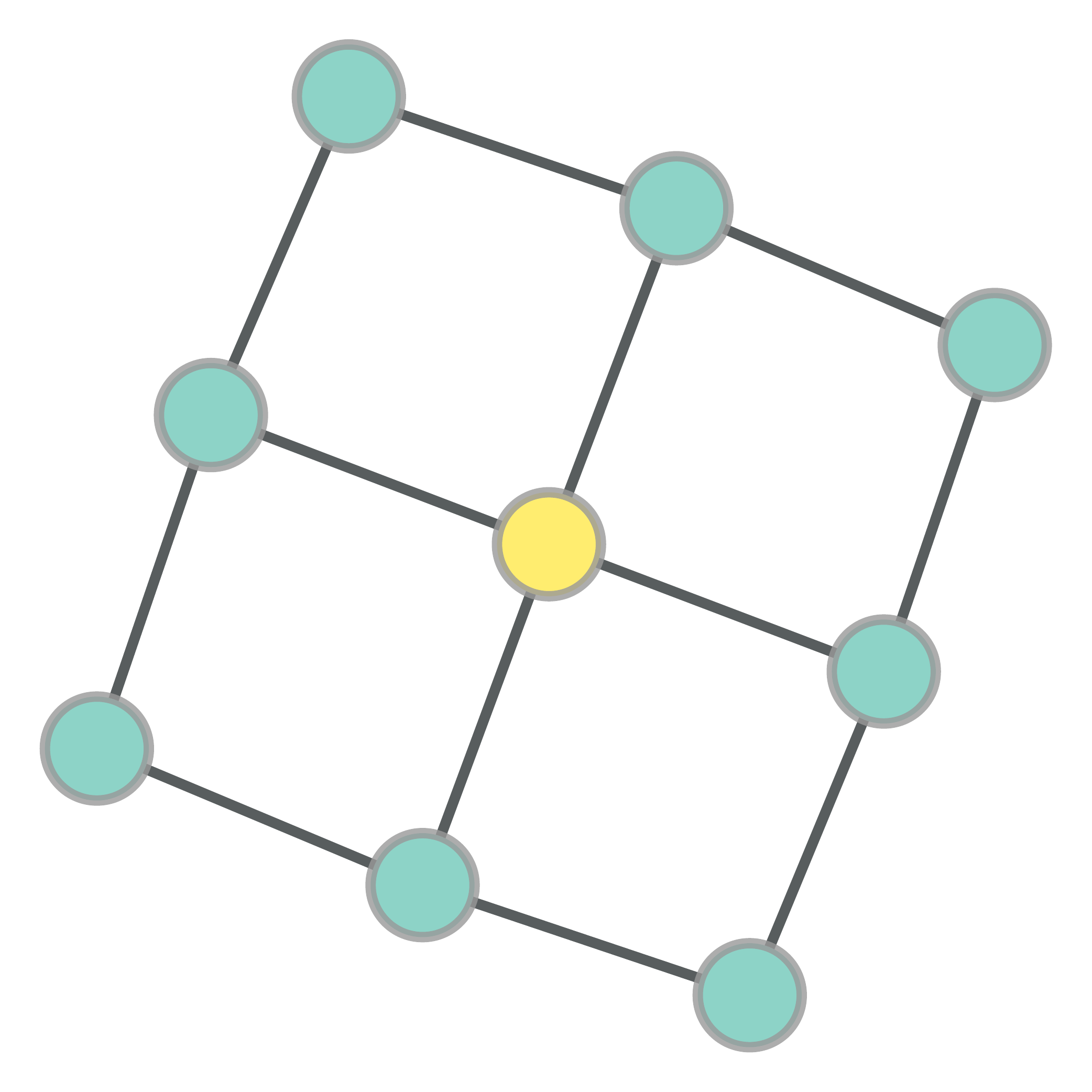}} &
\subfloat[]{\includegraphics[width = 0.4in]{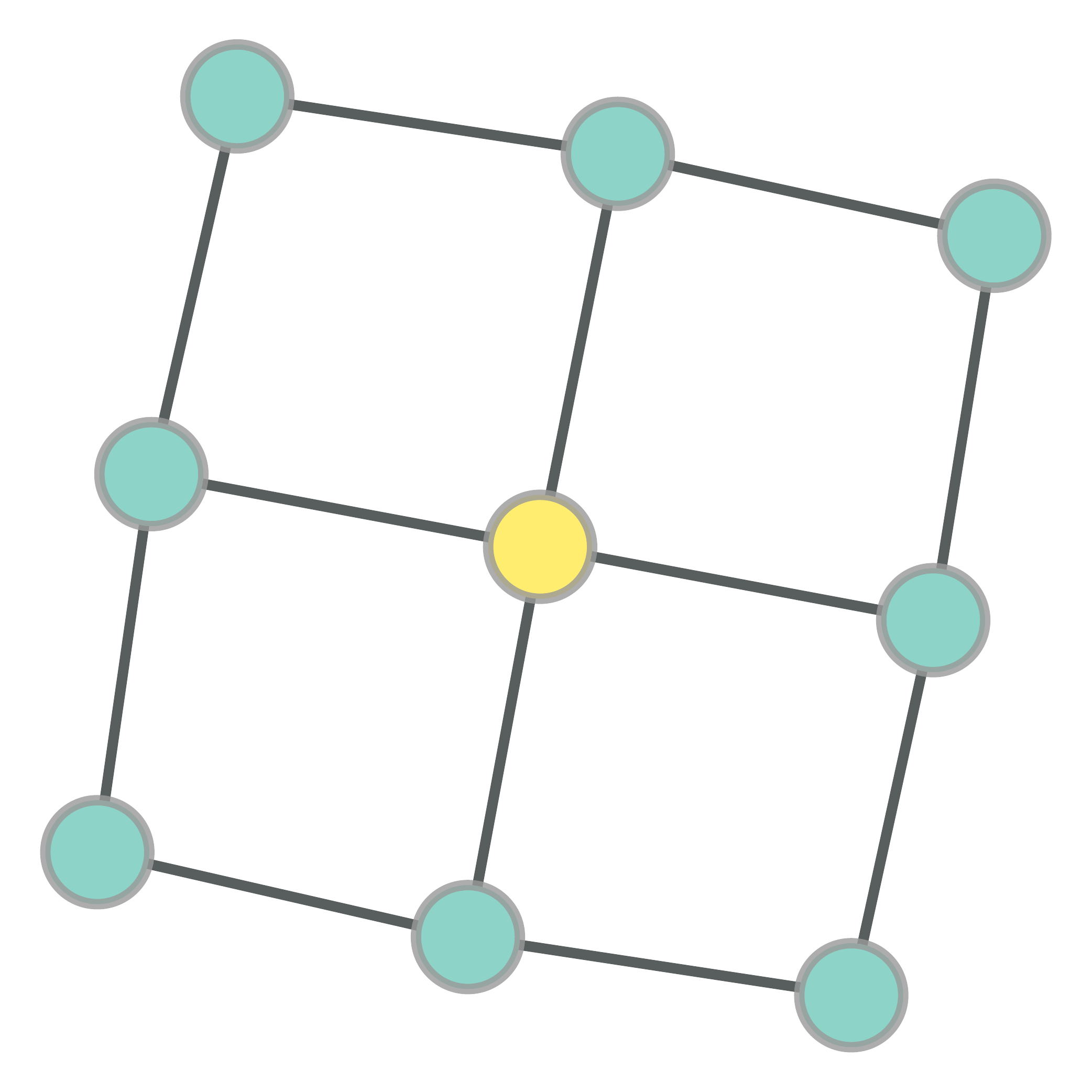}} &
 & 
\subfloat[]{\includegraphics[width = 0.8in]{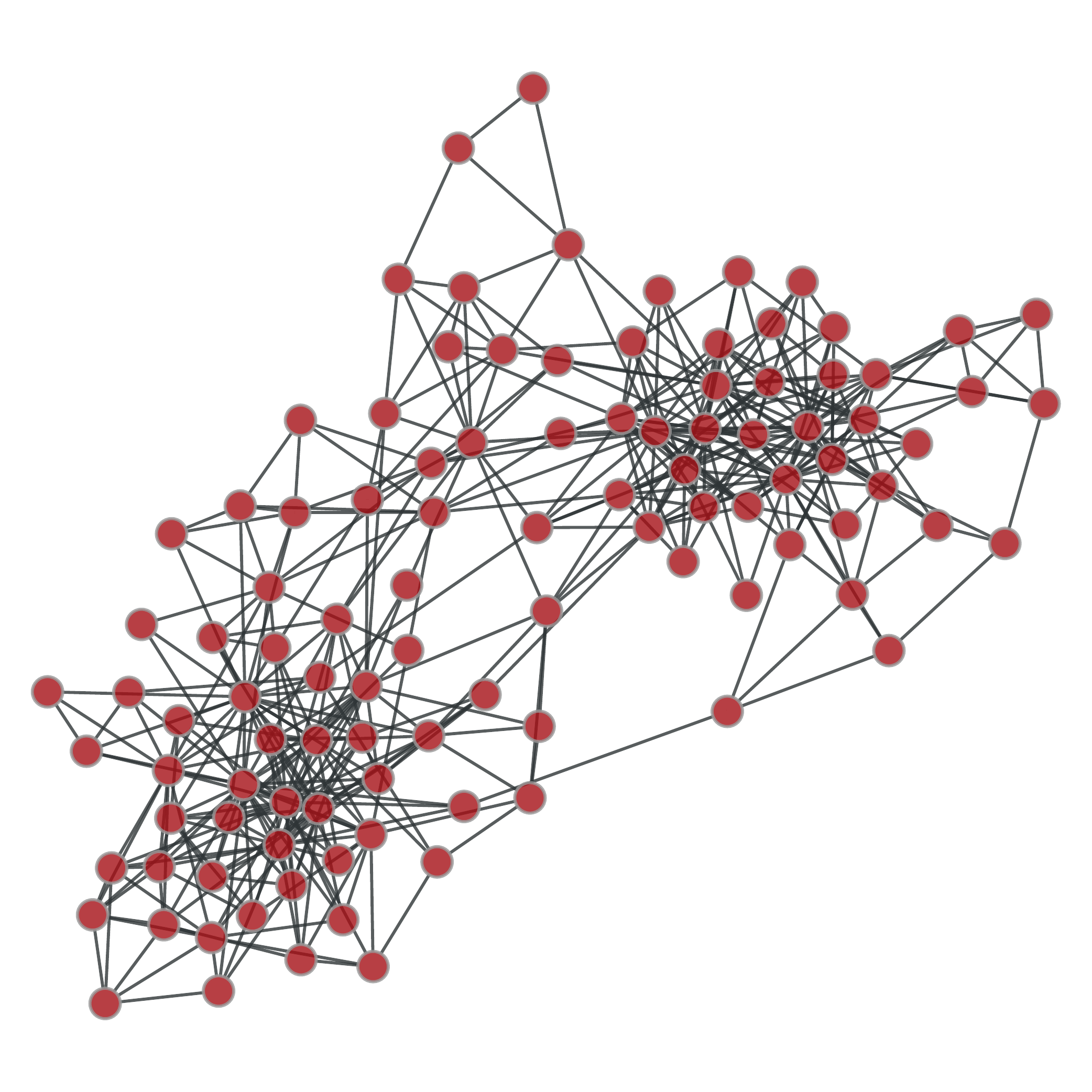}} &
\subfloat[]{\includegraphics[width = 0.7in]{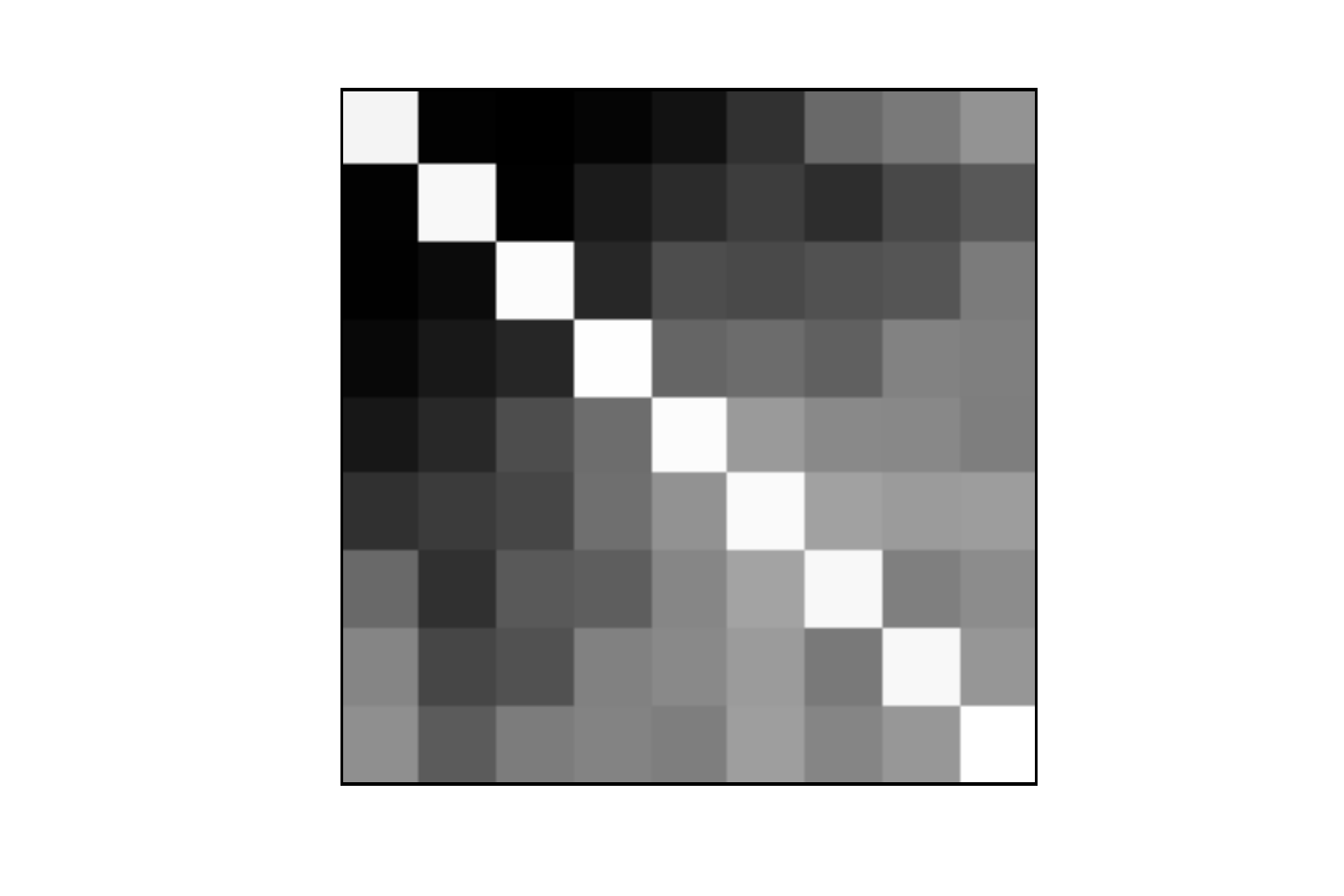}} &
\subfloat[]{\includegraphics[width = 0.4in]{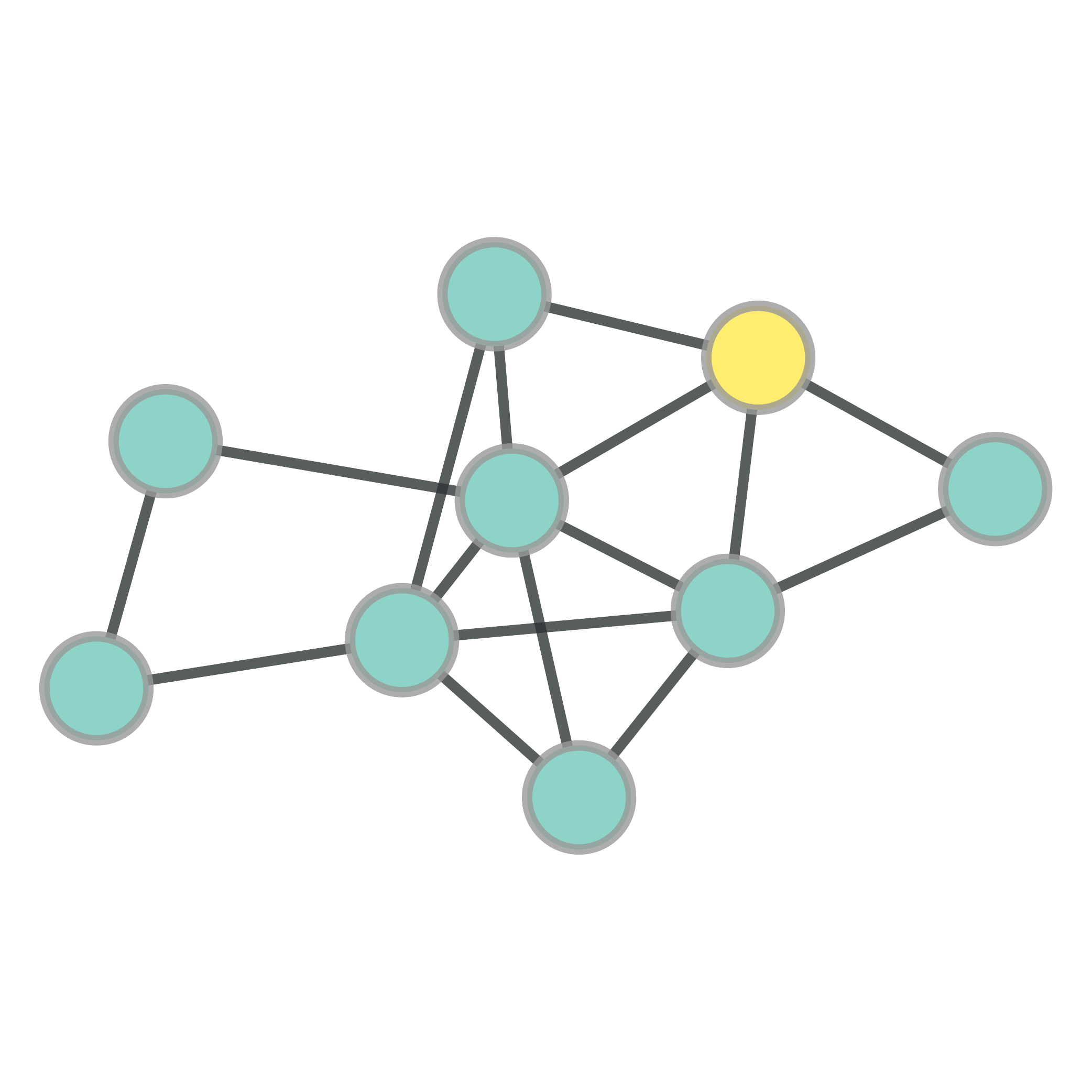}} &
\subfloat[]{\includegraphics[width = 0.4in]{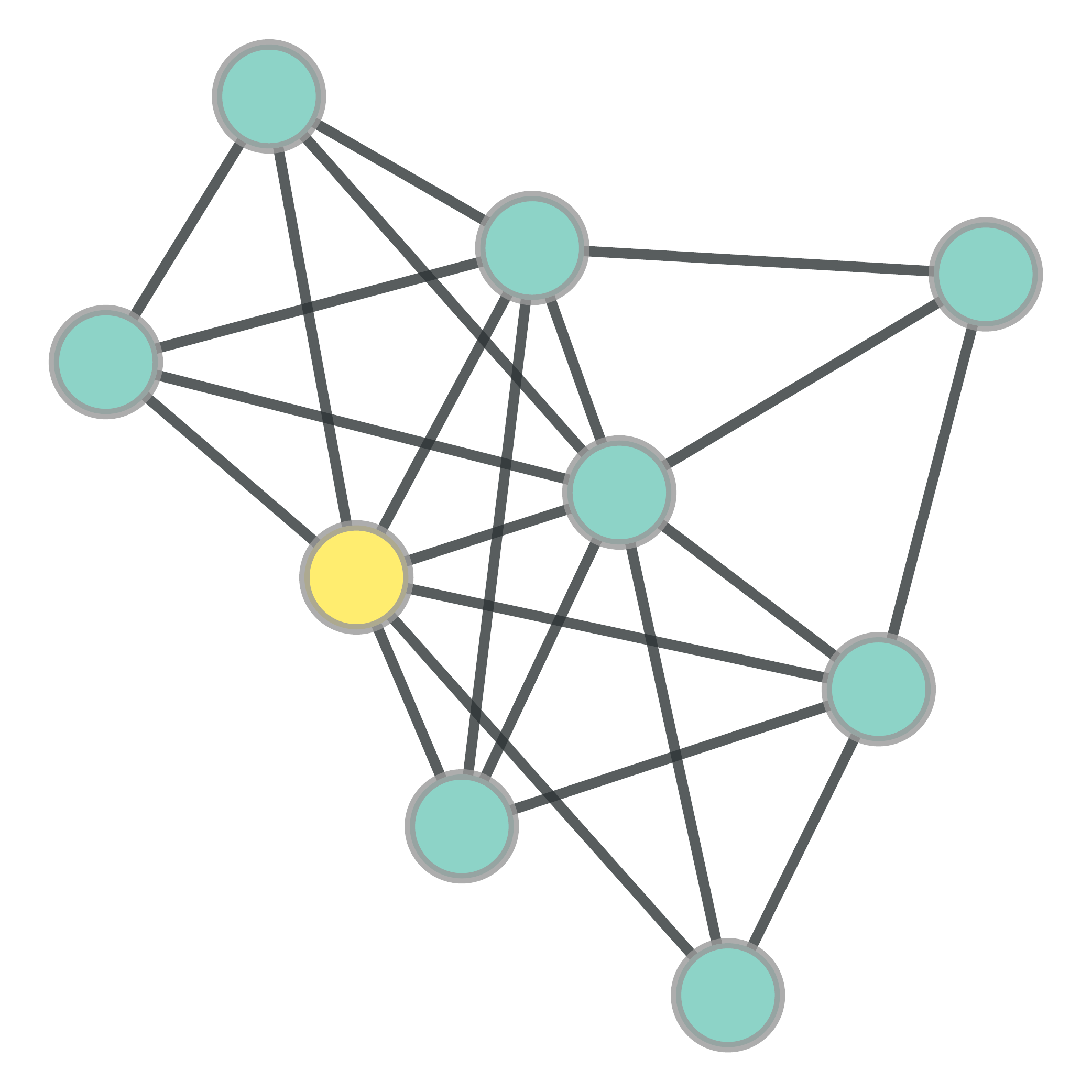}} &
\subfloat[]{\includegraphics[width = 0.4in]{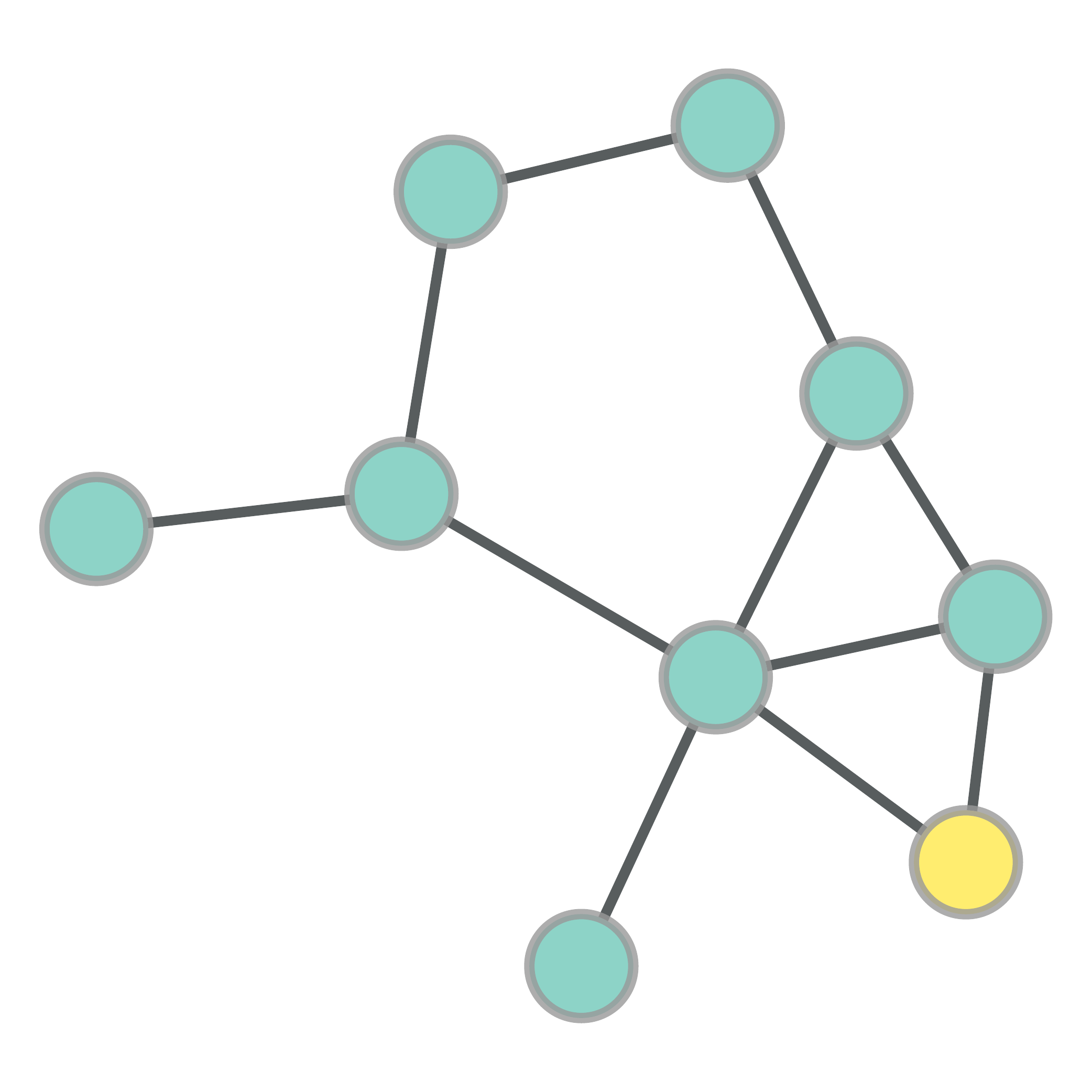}} \\[-6.5ex]

\multirow{2}{*}[1.8em]{\subfloat[]{\includegraphics[width = 0.8in]{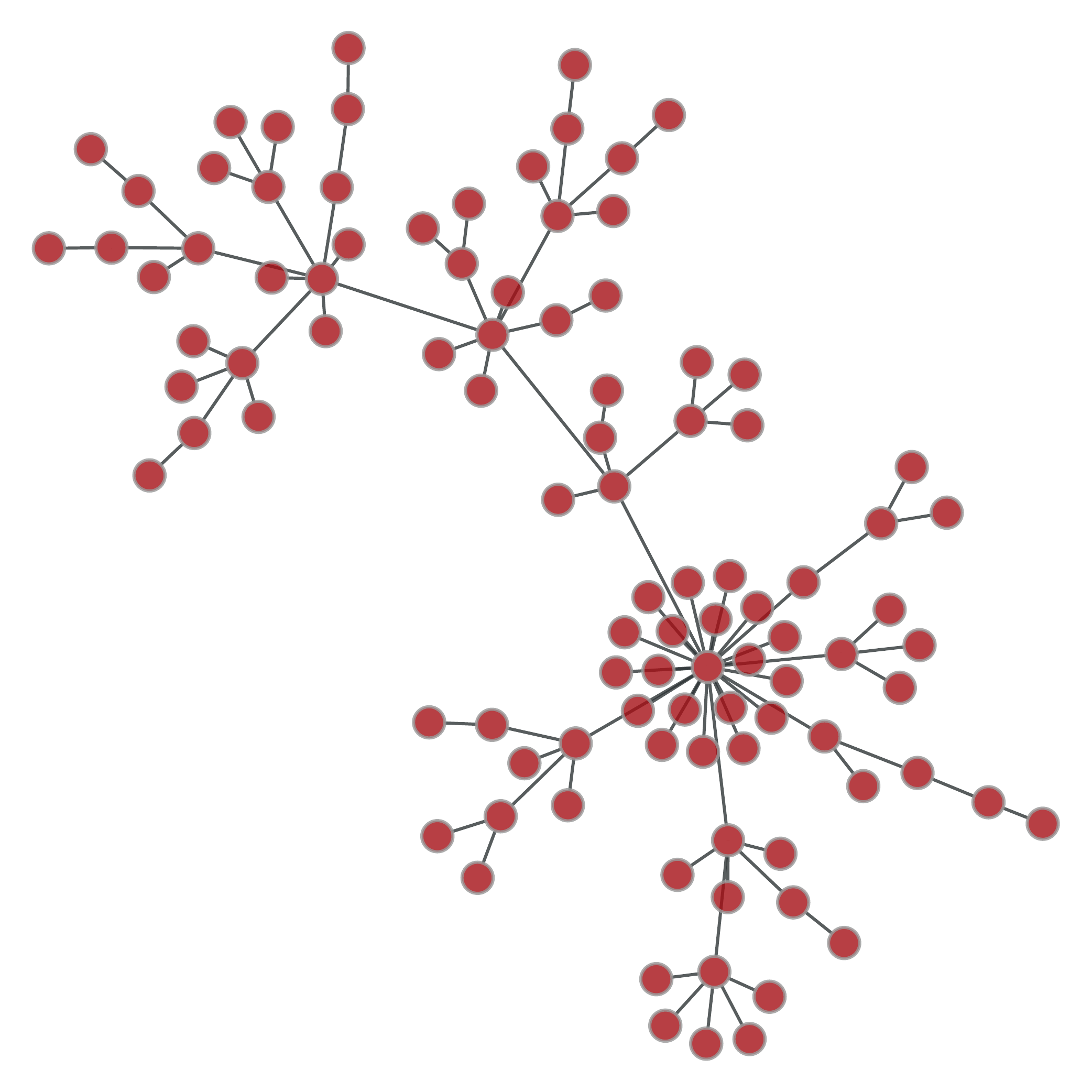}}} &
\subfloat[]{\includegraphics[width = 0.7in]{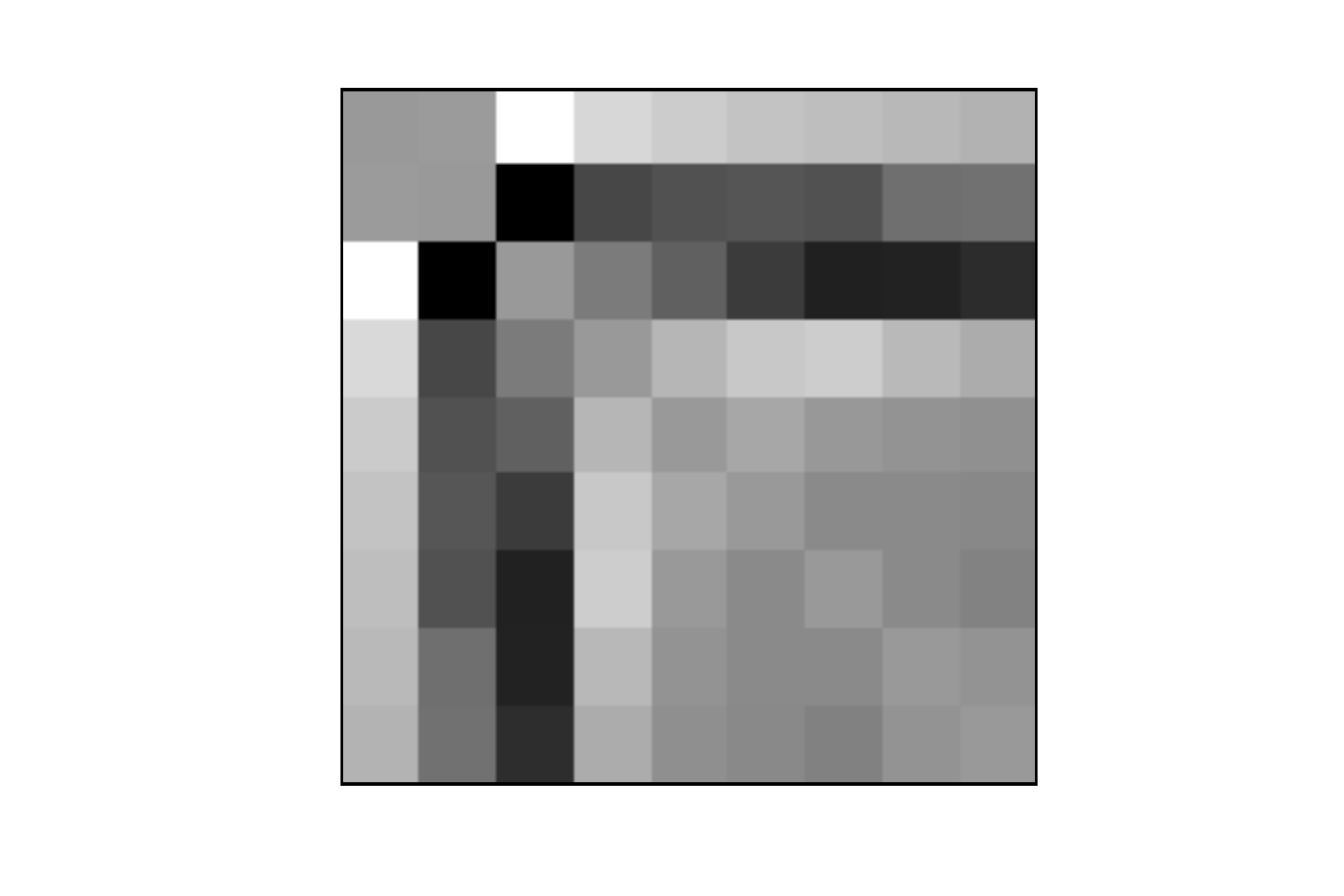}} &
\subfloat[]{\includegraphics[width = 0.4in]{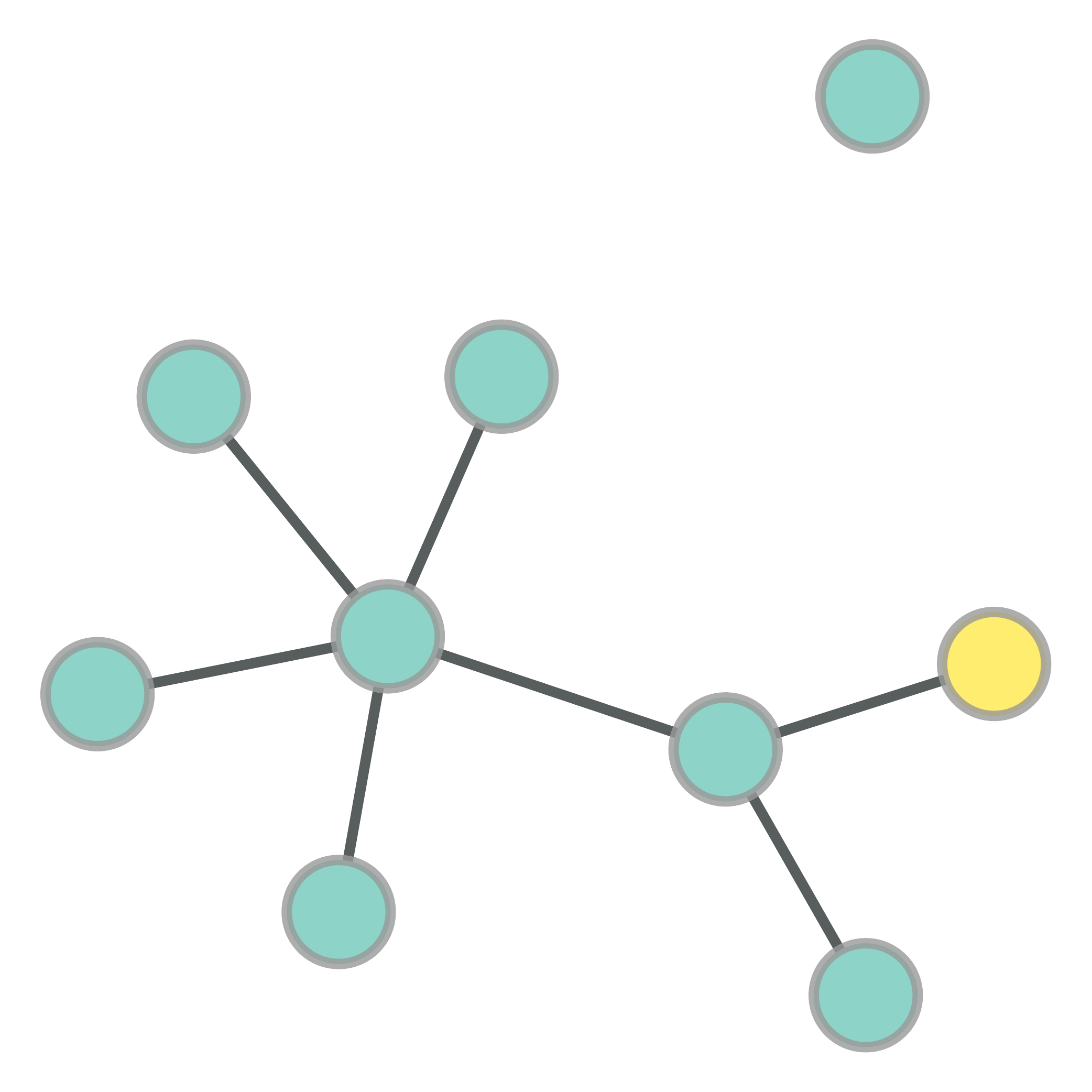}} &
\subfloat[]{\includegraphics[width = 0.4in]{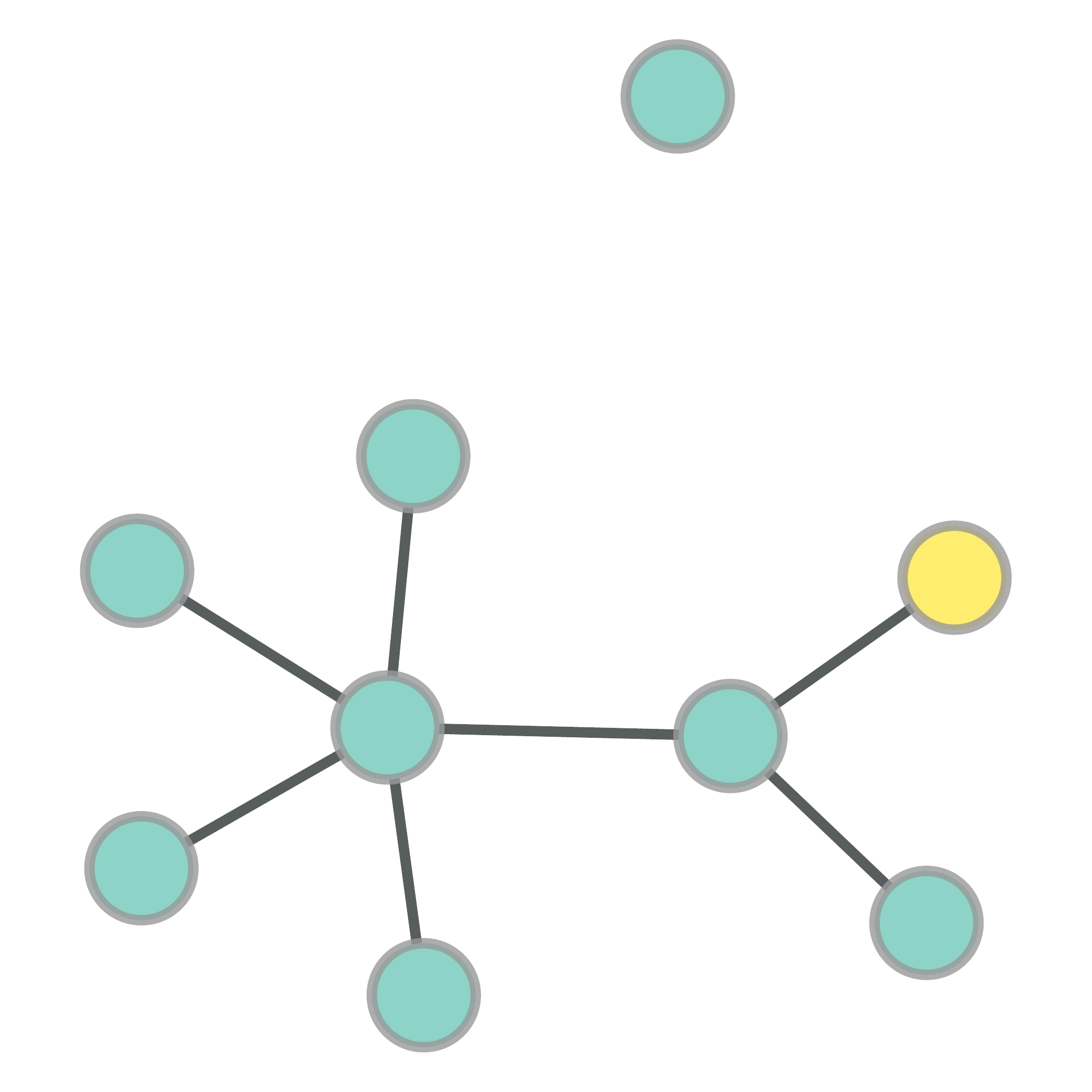}} &
\subfloat[]{\includegraphics[width = 0.4in]{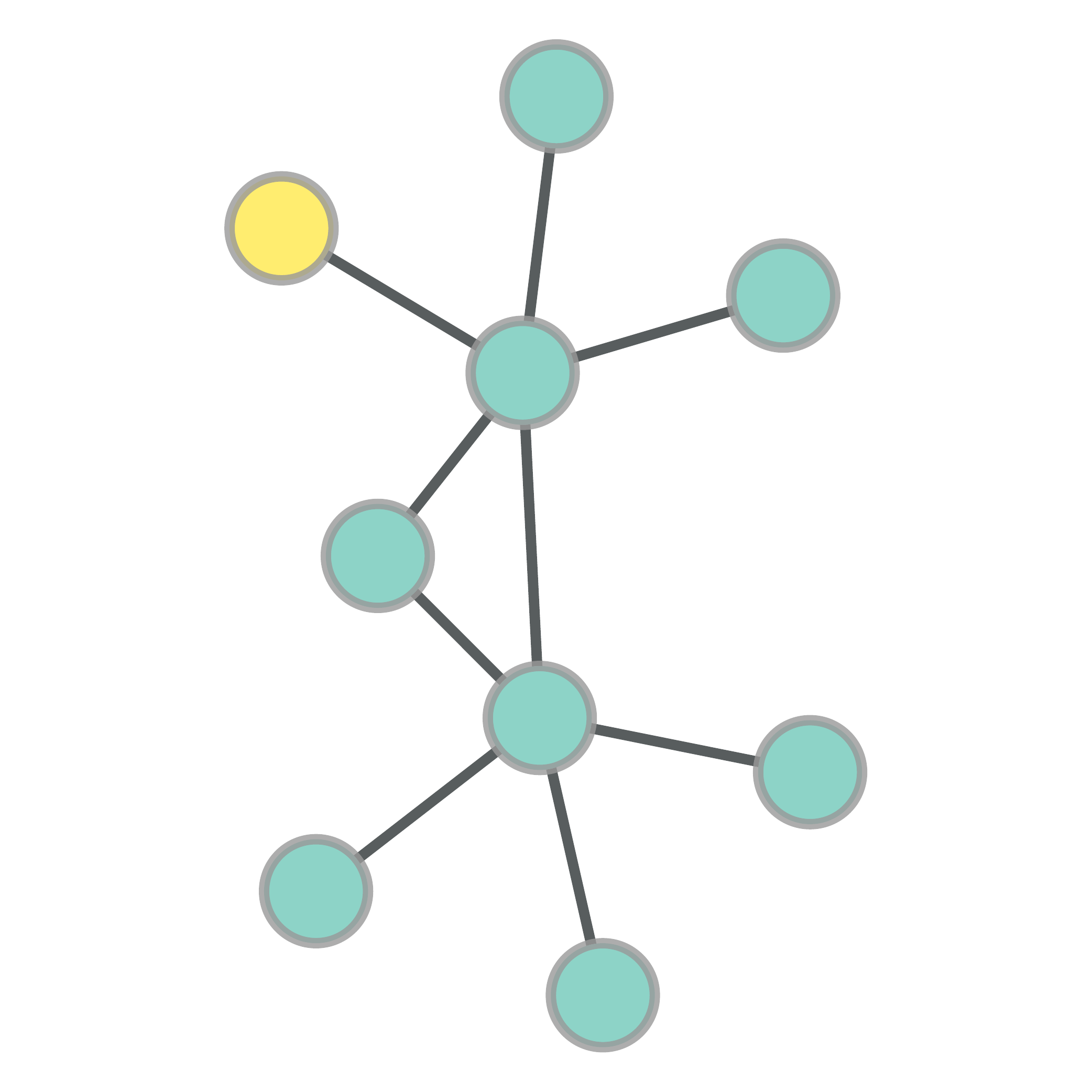}} &
  &
\multirow{2}{*}[2em]{\subfloat[]{\includegraphics[width = 0.8in]{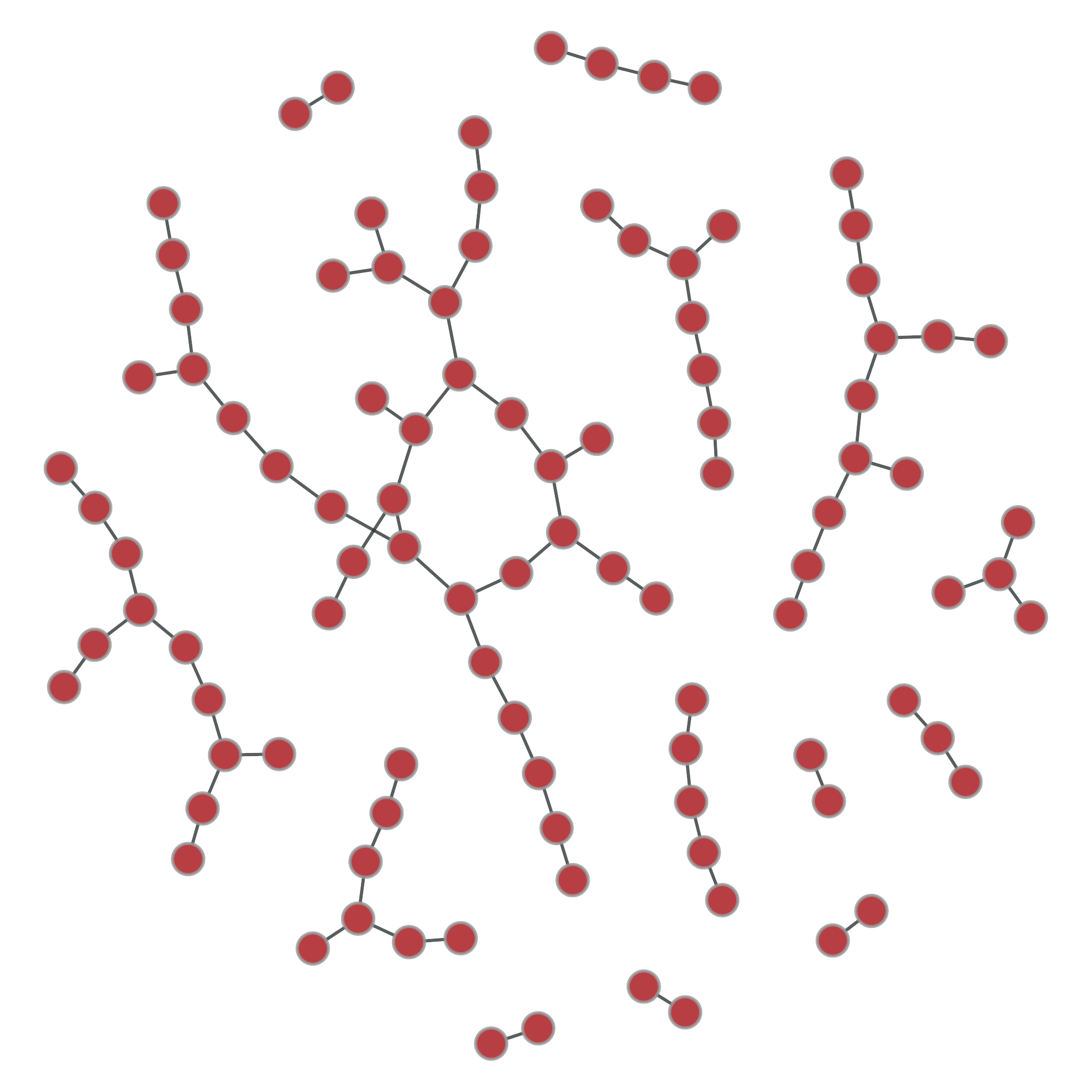}}} &
\subfloat[]{\includegraphics[width = 0.7in]{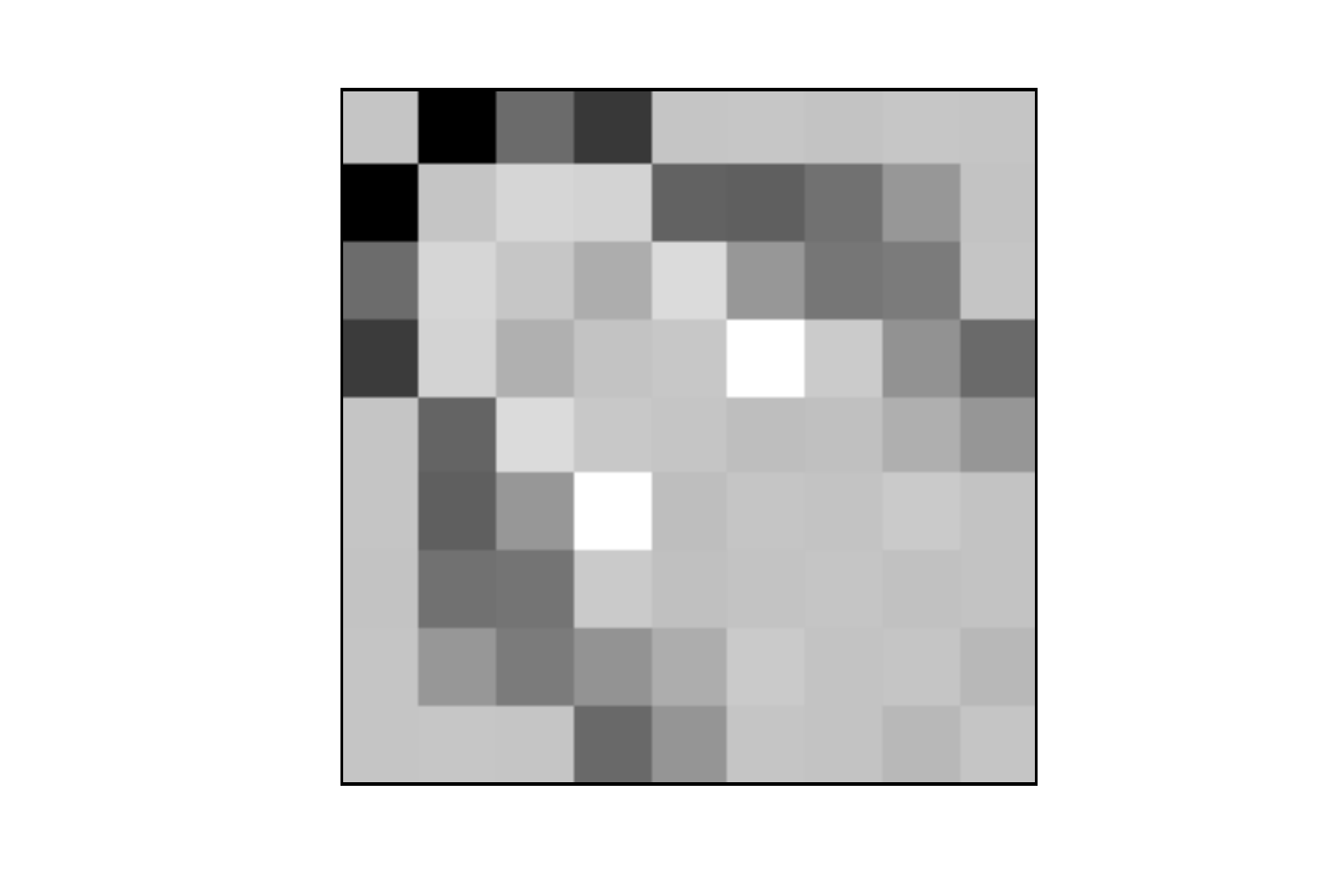}} &
\subfloat[]{\includegraphics[width = 0.4in]{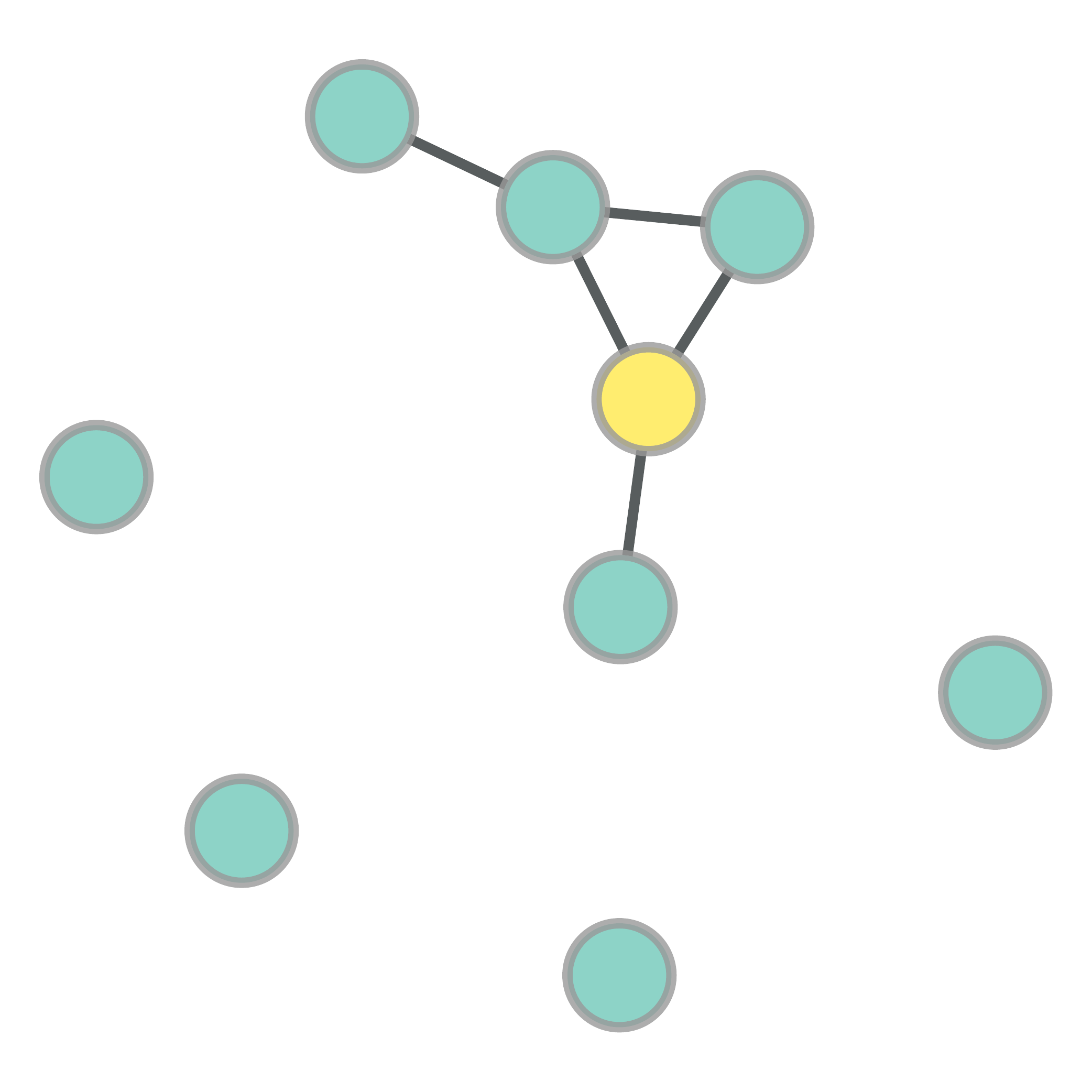}} &
\subfloat[]{\includegraphics[width = 0.4in]{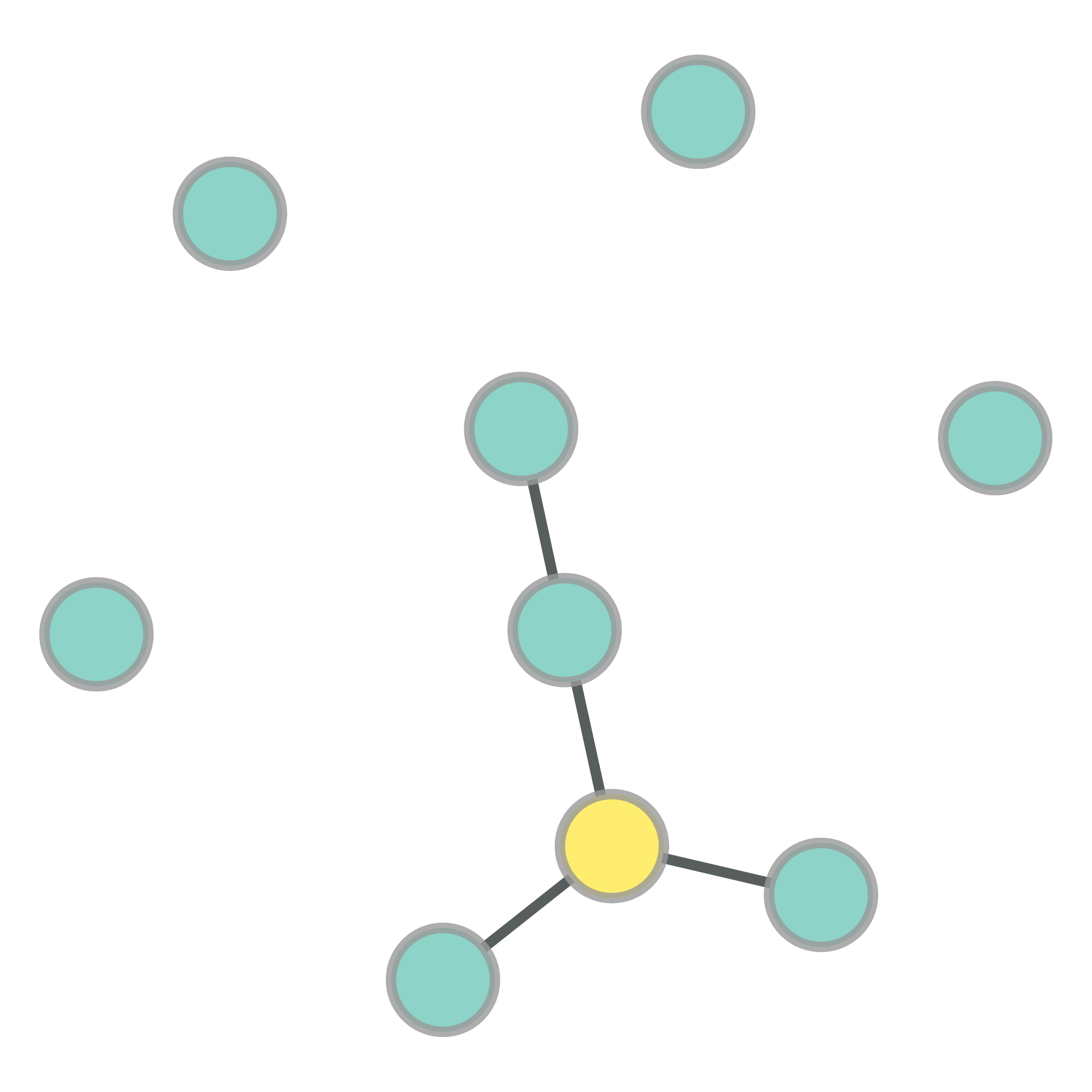}} &
\subfloat[]{\includegraphics[width = 0.4in]{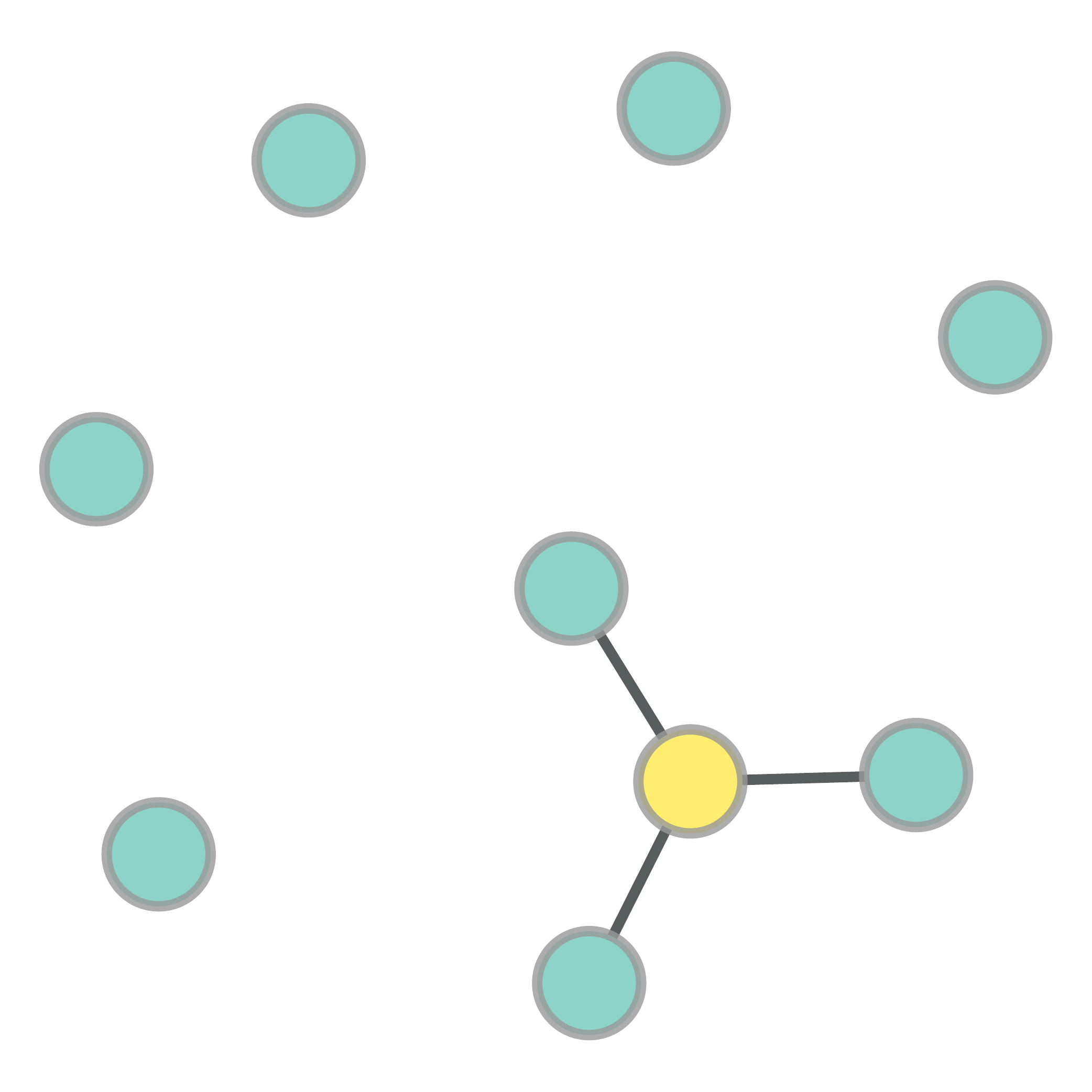}} \\[-5.8ex]
  &
\subfloat[]{\includegraphics[width = 0.7in]{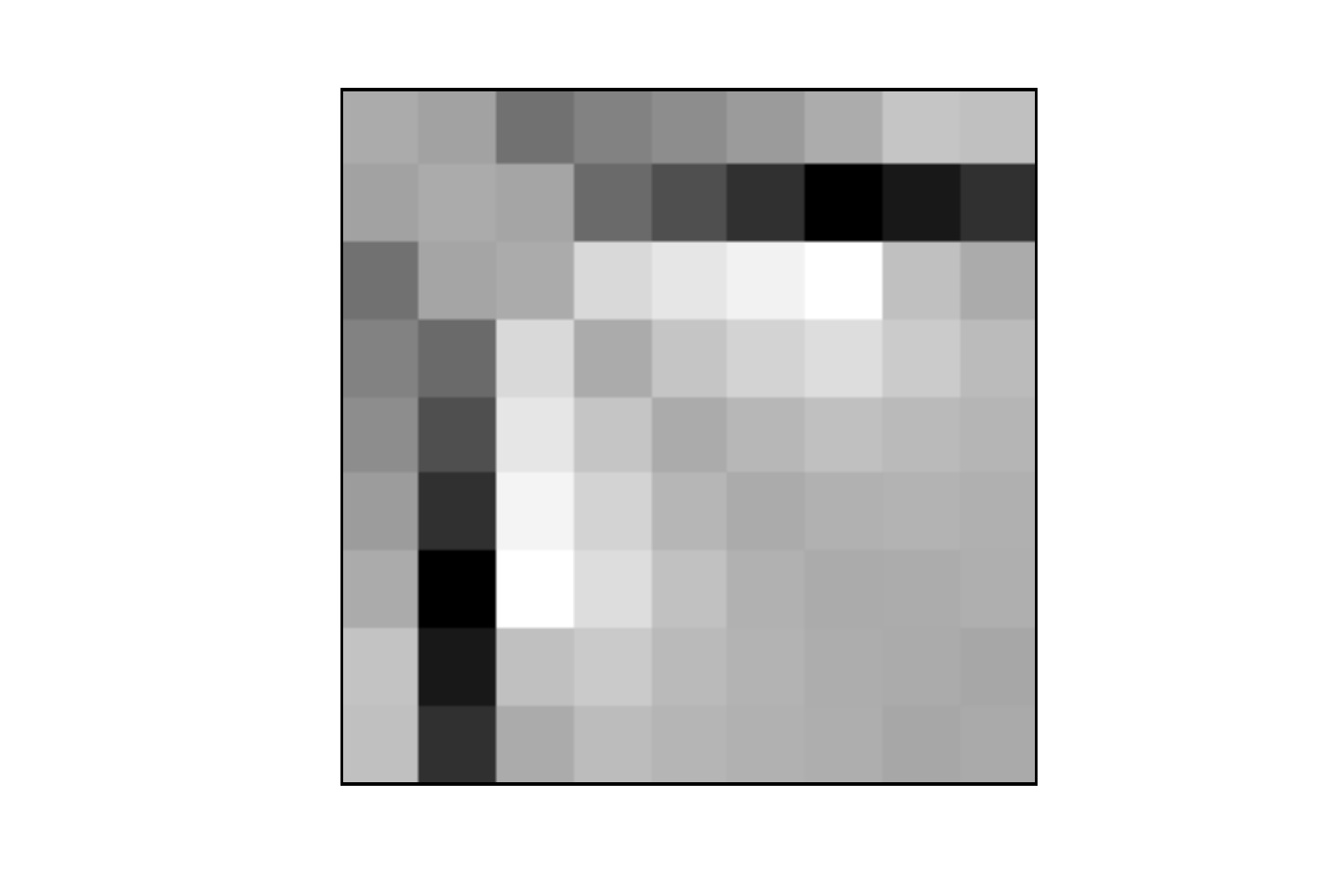}} &
\subfloat[]{\includegraphics[width = 0.4in]{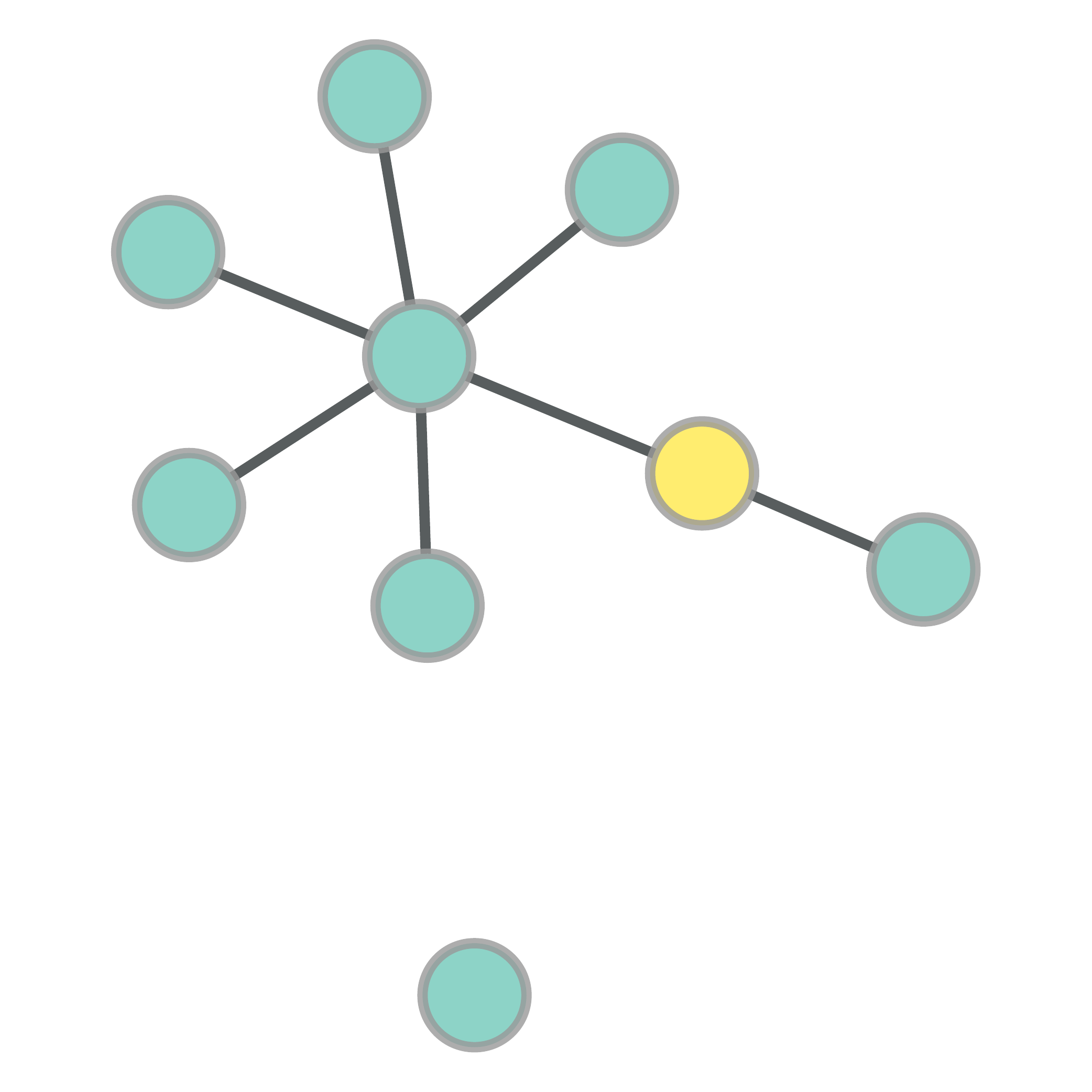}} &
\subfloat[]{\includegraphics[width = 0.4in]{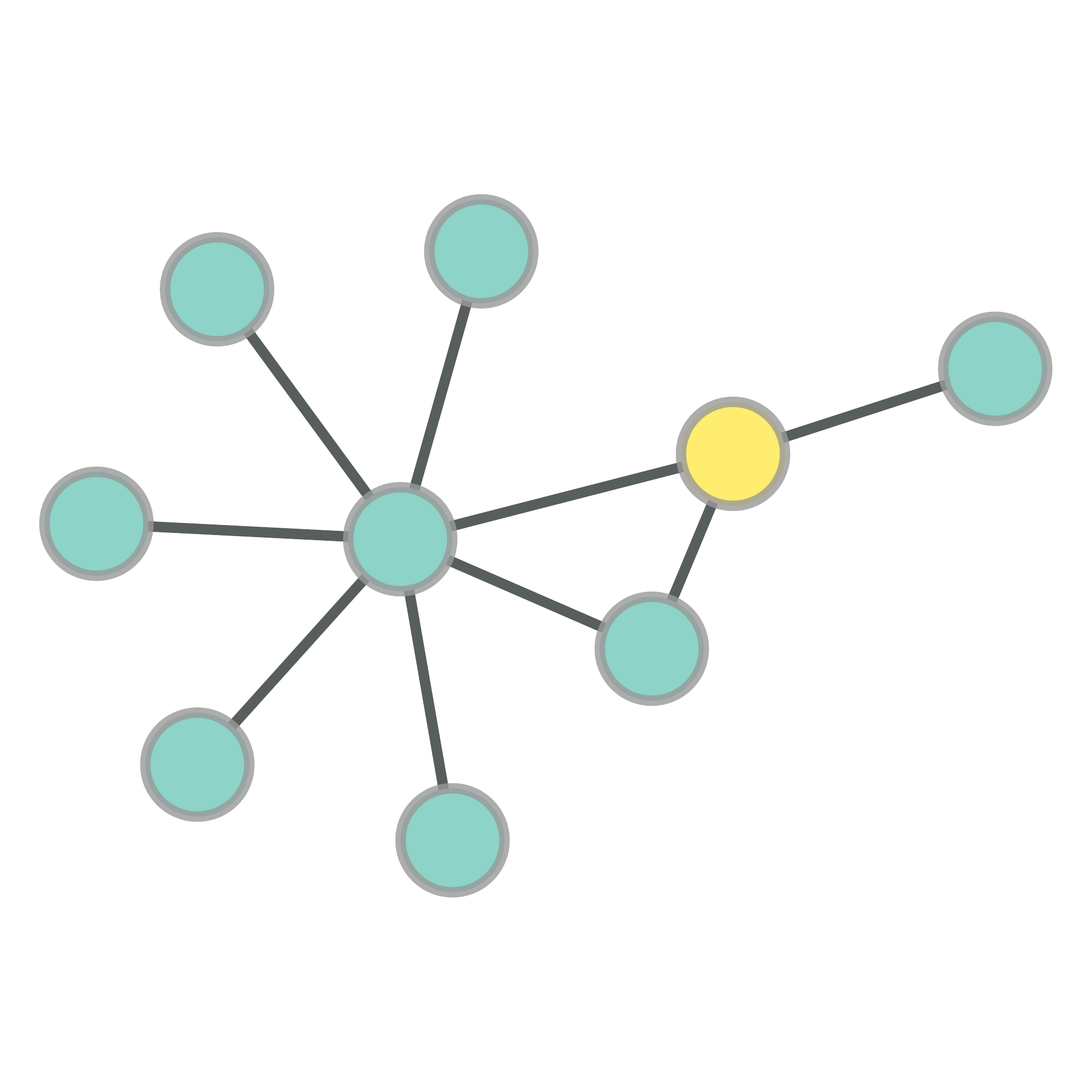}} &
\subfloat[]{\includegraphics[width = 0.4in]{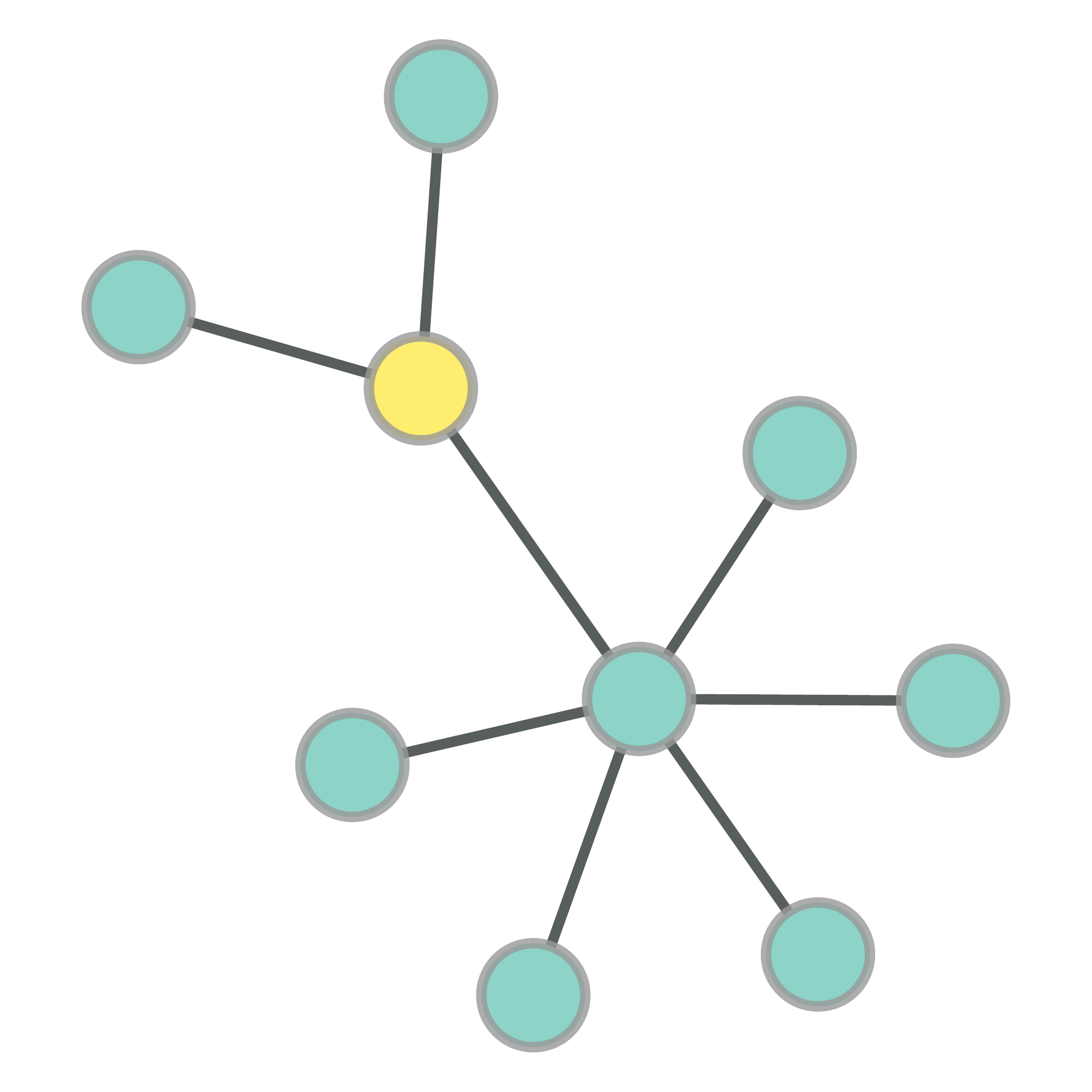}} &
  &
  & 
\subfloat[]{\includegraphics[width = 0.7in]{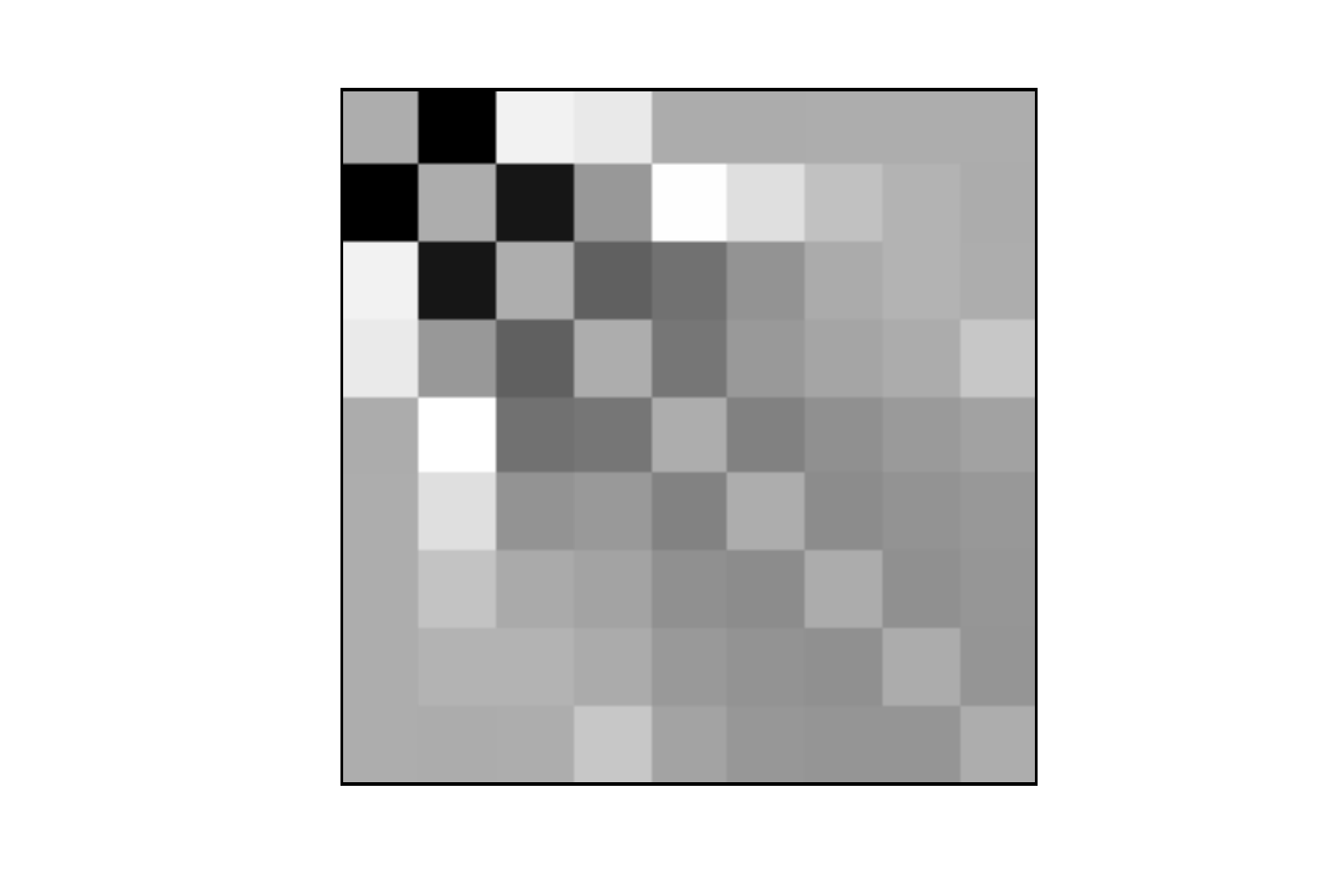}} &
\subfloat[]{\includegraphics[width = 0.4in]{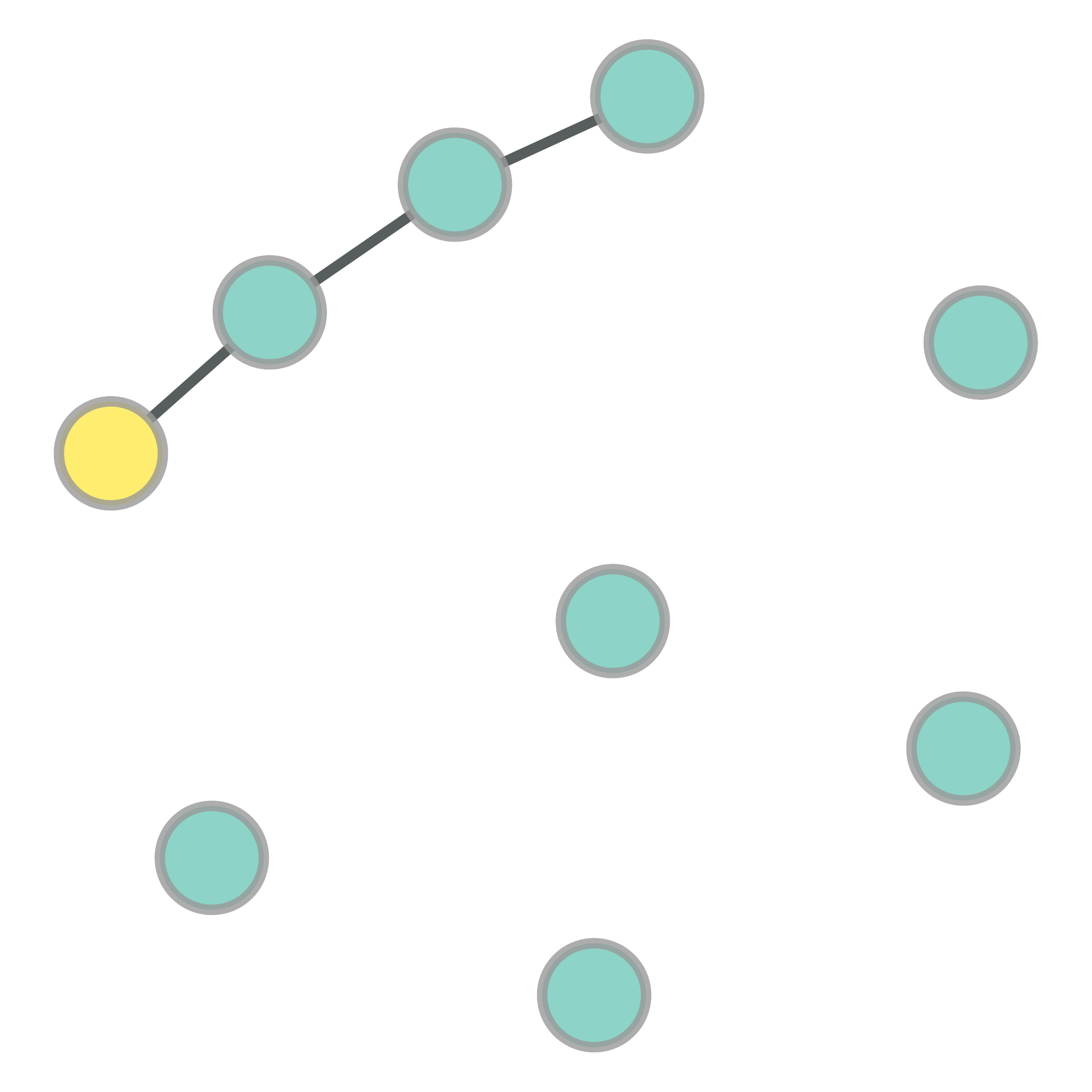}} &
\subfloat[]{\includegraphics[width = 0.4in]{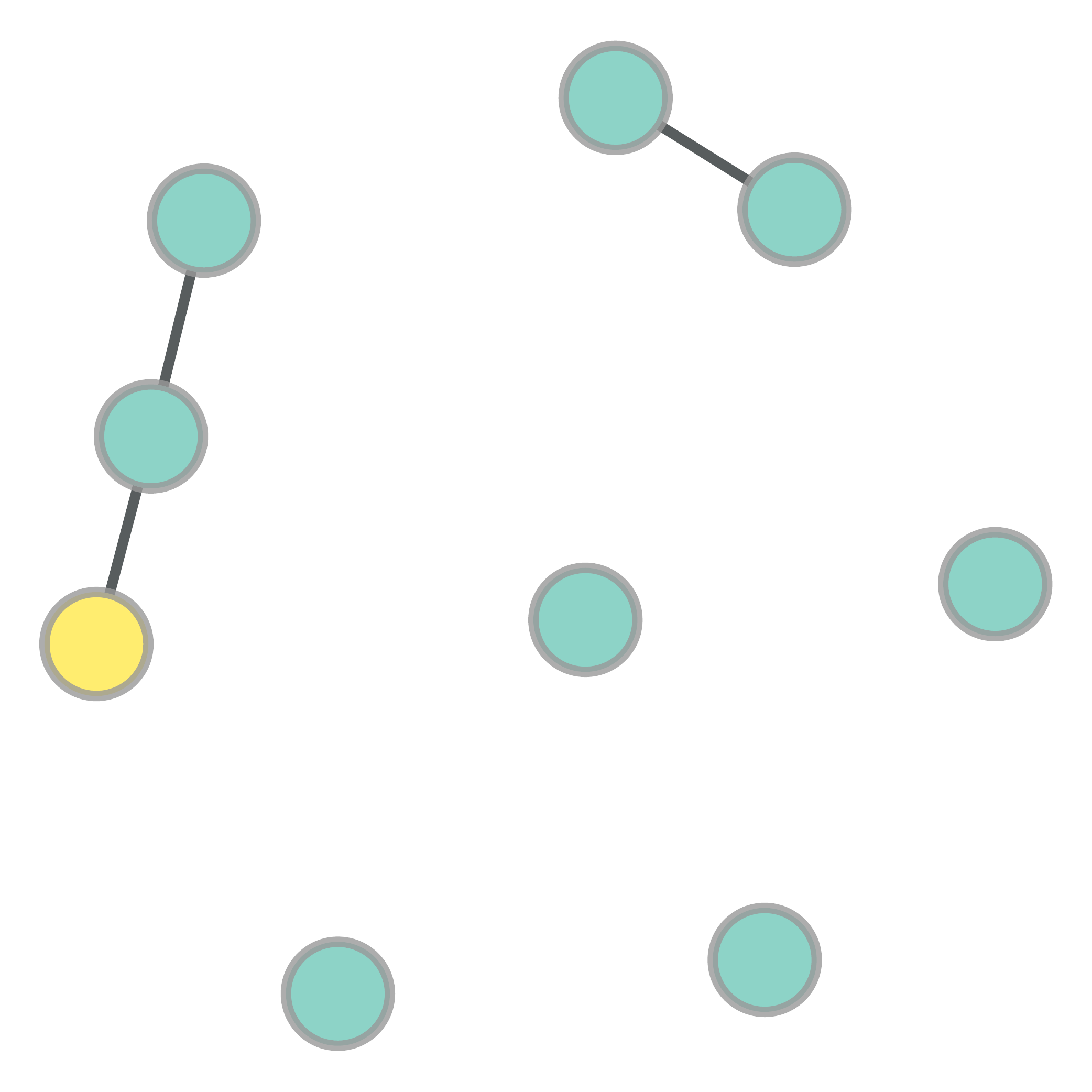}} &
\subfloat[]{\includegraphics[width = 0.4in]{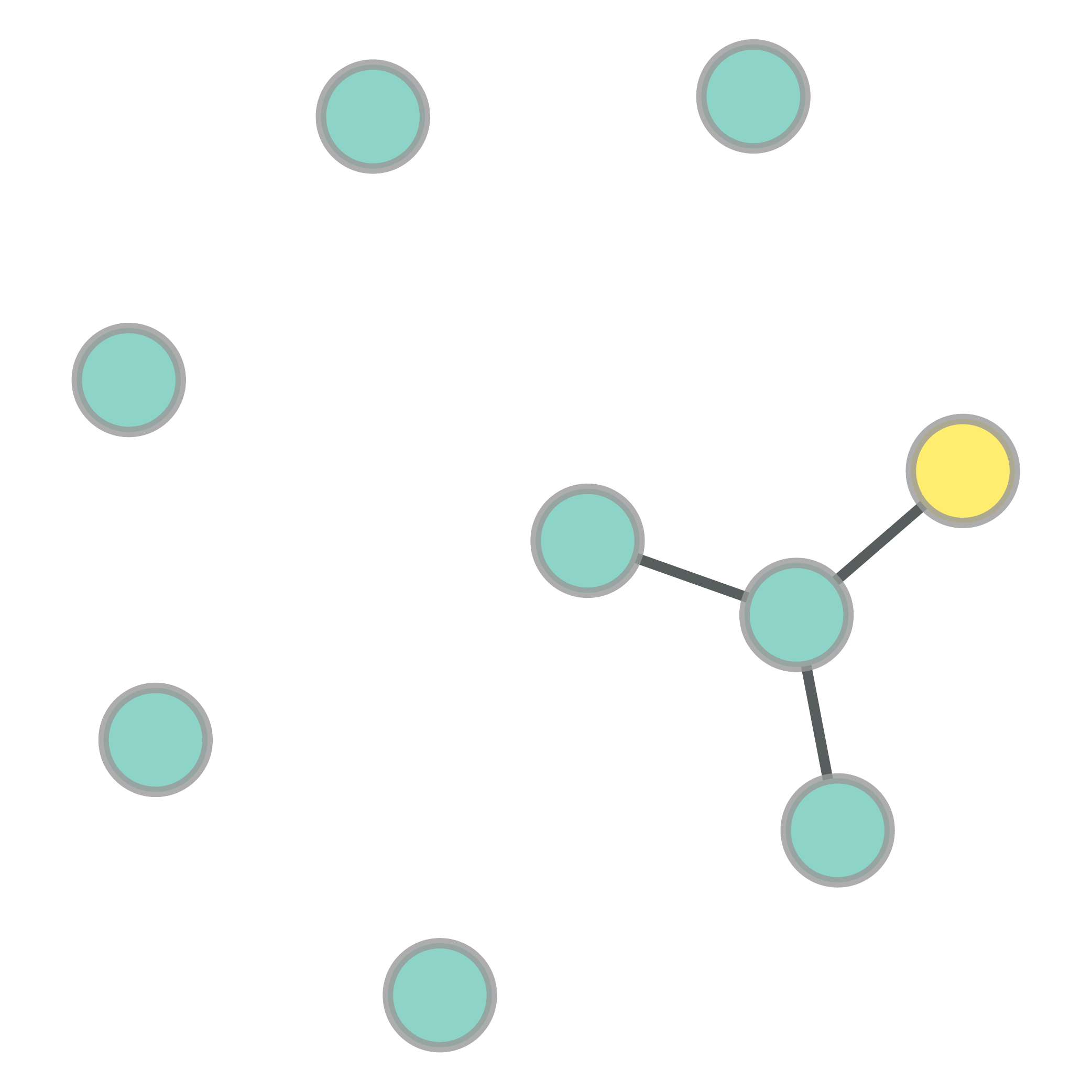}} \\
\end{tabular}
\vspace{-2mm}
\caption{\label{fig-features}Visualization of RBM features learned with 1-dimensional WL normalized receptive fields of size $9$ for a torus (periodic lattice, top left), a preferential attachment graph (\citealt{Barabasi:1999}, bottom left), a co-purchasing network of political books (top right), and a random graph (bottom right). Instances of these graphs with about $100$ nodes are depicted on the left. A visual representation of the feature's weights (the darker a pixel, the stronger the corresponding weight) and $3$ graphs sampled from the RBMs by setting all but the hidden node corresponding to the feature to zero. Yellow nodes have position $1$ in the adjacency matrices. (Best seen in color.)}
\end{figure*}

\subsection{Runtime Analysis}
We assess the efficiency of \textsc{Patchy-san} by applying it to real-world graphs. The objective is to compare the rates at which receptive fields are generated to the rate at which state of the art CNNs perform learning. All input graphs are part of the collection of the Python module \textsc{graph-tool}\footnote{https://graph-tool.skewed.de/}. For a given graph, we used \textsc{Patchy-san} to compute a receptive field for \emph{all}  nodes using the $1$-dimensional Weisfeiler-Lehman~\cite{douglas2011weisfeiler} (1-WL) algorithm for the normalization.  \textbf{torus} is a periodic lattice with $10,000$ nodes; \textbf{random} is a random undirected graph with $10,000$ nodes and a degree distribution $P(k)  \propto 1/k$ and $k_{\max} = 3$; \textbf{power} is a network representing the topology of a power grid in the US; \textbf{polbooks} is a co-purchasing network of books about US politics published during the $2004$ presidential election; \textbf{preferential} is a preferential attachment network model where newly added vertices have degree $3$; \textbf{astro-ph} is a coauthorship network between authors of preprints posted on the astrophysics arxiv~\cite{newman:2001}; \textbf{email-enron} is a communication network generated from about half a million sent emails~\cite{leskovec:2009}. All experiments were run on commodity hardware with 64G RAM and a single 2.8 GHz CPU. 

Figure~\ref{fig-runtime} depicts the receptive fields per second rates for each input graph. For receptive field size $k=5$ and $k=10$ \textsc{Patchy-san} creates fields at a rate of more than $1000/s$ except for \textbf{email-enron} with a rate of $600/s$ and $320/s$, respectively. For $k=50$, the largest tested size, fields are created at a rate of at least $100/s$. A CNN with $2$ convolutional and $2$ dense layers learns at a rate of about $200$-$400$ training examples per second on the same machine. Hence, the speed at which receptive fields are generated is sufficient to saturate a downstream CNN.

\subsection{Feature Visualization}

The visualization experiments' aim is to qualitatively investigate whether popular models such as the restricted Boltzman machine (RBM)~\cite{freund:1992} can be combined with \textsc{Patchy-san} for unsupervised feature learning.  For every input graph, we have generated receptive fields for all nodes and used these as input  to an RBM.  The RBM had $100$ hidden nodes and was trained for $30$ epochs with contrastive divergence and a learning rate of $0.01$. We visualize the features learned by a single-layer RBM for $1$-dimensional Weisfeiler-Lehman (1-WL) normalized receptive fields of size $9$. Note that the features learned by the RBM correspond to reoccurring receptive field patterns. Figure~\ref{fig-features} depicts some of the features and samples drawn from it for four different graphs.

%

\begin{table*}[t!]
\footnotesize
\centering
\begin{tabular}{l|c|c|c|c|c}
{\bf Data set} & MUTAG & PCT & NCI1 & PROTEIN & D \& D \\ 
\hline
\hline 
Max  & 28 & 109 & 111 & 620 & 5748 \\ 
\hline 
Avg  & 17.93 & 25.56 & 29.87 &   39.06 & 284.32 \\ 
\hline 
Graphs & 188 & 344 & 4110 &   1113 & 1178 \\ 
\hline 
\hline 
SP \ \ [\citenum{Borgwardt:2005}] & $85.79 \pm 2.51$ & $58.53 \pm 2.55$ & $73.00 \pm 0.51$ &  $75.07 \pm 0.54$  & $> 3$ days \\ 
\hline 
RW [\citenum{Gaertner:2003}] & $83.68 \pm 1.66$  & $57.26 \pm 1.30$ &   $> 3$ days   & $74.22 \pm 0.42$ & $> 3$ days \\ 
\hline 
GK \ [\citenum{Shervashidze:2009}]  & $81.58 \pm 2.11$ & $57.32 \pm 1.13$ & $62.28 \pm 0.29$ &  $71.67 \pm 0.55$ & $78.45 \pm 0.26$ \\ 
\hline 
WL [\citenum{Shervashidze:2011}] & $80.72 \pm 3.00\ (5s)$ & $56.97 \pm 2.01\ (30s)$ & $80.22 \pm 0.51\ (375s)$ & $72.92 \pm 0.56\ (143s)$ & $77.95 \pm 0.70\ (609s)$ \\ 
\hline
\hline
PSCN $k$=$5$ & $91.58 \pm 5.86\ (2s)$ & $59.43 \pm 3.14\ \ (4s)$  & $72.80 \pm 2.06\ \ (59s)$ &  $74.10 \pm 1.72\ \ (22s)$    &  $74.58 \pm 2.85\ (121s)$ \\ 
\hline
PSCN $k$=$10$ & $88.95 \pm 4.37 \ (3s)$ & $62.29 \pm 5.68\ \ (6s)$  & $76.34 \pm 1.68\ \ (76s)$ &   $75.00 \pm 2.51\ \ (30s)$    & $76.27 \pm 2.64\ (154s)$ \\ 
\hline
PSCN $k$=$10^{\mathtt{E}}$\hspace{-1mm} & $92.63 \pm 4.21 \ (3s)$ & $60.00 \pm 4.82\ \ (6s)$  & $78.59 \pm 1.89\ \ (76s)$ &   $75.89 \pm 2.76\ \ (30s)$    & $77.12 \pm 2.41\ (154s)$ \\ 
\hline 
\hline
PSLR $k$=$10$ & $87.37 \pm 7.88$ & $58.57 \pm 5.46$  & $70.00 \pm 1.98$ &   $71.79 \pm 3.71$ & $68.39 \pm 5.56$ \\ 
\hline
\end{tabular}
\caption{\label{fig-classification} Properties of the data sets and  accuracy and timing results (in seconds) for \textsc{patchy-san} and $4$ state of the art graph kernels. 
}
\end{table*}

\subsection{Graph Classification}
Graph classification is the problem of assigning graphs to one of several categories. 

{\bf Data Sets.}
We use $6$ standard benchmark data sets to compare run-time and classification accuracy with state of the art graph kernels: MUTAG, PCT, NCI1, NCI109, PROTEIN, and D\&D. MUTAG~\cite{debnath:1991} is a data set of $188$ nitro compounds where classes indicate whether the compound has a mutagenic effect on a bacterium. PTC consists of $344$ chemical compounds where classes indicate carcinogenicity for male and female rats~\cite{toivonen:2003}. NCI1 and NCI109 are chemical compounds screened for activity against non-small cell lung cancer and ovarian cancer cell lines~\cite{wale:2006}. PROTEINS is a graph collection where nodes are secondary structure elements and edges  indicate neighborhood in the amino-acid sequence or in 3D space. Graphs are classified as enzyme or non-enzyme. D\&D is a data set of $1178$ protein structures~\cite{Dobson:2003} classified into enzymes and non-enzymes. 

{\bf Experimental Set-up.}
We compared \textsc{Patchy-san} with the shortest-path kernel (SP)~\cite{Borgwardt:2005}, the random walk kernel (RW)~\cite{Gaertner:2003}, the graphlet count kernel (GK)~\cite{Shervashidze:2009}, and the Weisfeiler-Lehman subtree kernel (WL)~\cite{Shervashidze:2011}. Similar to previous work~\cite{Yanardag:2015}, we set the height parameter of WL to $2$, the size of the graphlets for GK to $7$, and chose the decay factor for RW from $\{10^{-6}, 10^{-5}, ..., 10^{-1}\}$.
We performed $10$-fold cross-validation with \textsc{LIB-SVM}~\cite{Chang:2011}, using $9$ folds for training and $1$ for testing, and repeated the experiments $10$ times. We report average prediction accuracies and standard deviations.

For \textsc{Patchy-san} (referred to as PSCN), we used $1$-dimensional WL normalization, a width $w$ equal to the average number of nodes (see Table~\ref{fig-classification}), and receptive field sizes of $k=5$ and $k=10$. For the experiments we only used node attributes. In addition, we ran experiments for $k=10$ where we combined receptive fields for nodes and edges using a merge layer ($k=10^{\mathtt{E}}$). To make a fair comparison, we used a single network architecture with two convolutional layers, one dense hidden layer, and a softmax layer for all experiments. The first convolutional layer had $16$ output channels (feature maps). The second conv layer has $8$ output channels, a stride of $s=1$, and a  field size of $10$. The convolutional layers have rectified linear units. The dense layer has $128$ rectified linear units with a dropout rate of $0.5$. Dropout and the relatively small number of neurons are needed to avoid overfitting on the smaller data sets. The only hyperparameter we optimized is the number of epochs and the batch size for the mini-batch gradient decent algorithm \textsc{rmsprop}. All of the above was implemented with the \textsc{Theano}~\cite{bergstra:2010} wrapper \textsc{Keras}~\cite{chollet:2015}. We also applied a logistic regression (PSLR) classifier on the patches for $k=10$. 

Moreover, we ran experiments with the same set-up\footnote{Due to the larger size of the data sets, we removed dropout.} on larger social graph data sets (up to $12000$ graphs each, with an average of $400$ nodes), and compared \textsc{Patchy-san} with previously reported results for the graphlet count (GK) and the deep graphlet  count kernel (DGK)~\cite{Yanardag:2015}.  We used the normalized node degree as attribute for \textsc{Patchy-san}, highlighting one of its advantages: it can easily incorporate continuous features. 

\begin{table}[t!]
\footnotesize
\centering
\begin{tabular}{l|c|c|c}
{\bf Data set} & GK [\citenum{Shervashidze:2009}]  & DGK [\citenum{Yanardag:2015}]  & PSCN $k$=$10$ \\
\hline
\hline
COLLAB & $72.84 \pm 0.28$ & $73.09 \pm 0.25$ & $72.60 \pm 2.15$ \\
\hline
IMDB-B & $65.87 \pm 0.98$ & $66.96 \pm 0.56$ & $71.00 \pm 2.29$ \\
\hline
IMDB-M & $43.89 \pm 0.38$ & $44.55 \pm 0.52$ & $45.23 \pm 2.84$ \\
\hline
RE-B & $77.34 \pm 0.18$ & $78.04 \pm 0.39$ & $86.30 \pm 1.58$ \\
\hline
RE-M5k &  $41.01 \pm 0.17$ & $41.27 \pm 0.18$ & $49.10 \pm 0.70$ \\
\hline
RE-M10k & $31.82 \pm 0.08$  & $32.22 \pm 0.10$ & $41.32 \pm 0.42$\\
\hline
\end{tabular}
\caption{\label{table-social-graphs} Comparison of accuracy results on social graphs [\citenum{Yanardag:2015}].}
\end{table}

{\bf Results.}
Table~\ref{fig-classification} lists the results of the experiments. We omit the results for NCI109 as they are almost identical to NCI1. Despite using a one-fits-all CNN architecture, the CNNs accuracy is highly competitive with existing graph kernels. In most cases, a receptive field size of $10$ results in the best classification accuracy. The relatively high variance can be explained with the small size of the benchmark data sets and the fact that the CNNs hyperparameters (with the exception of epochs and batch size) were not tuned to individual data sets. Similar to the experience on image and text data, we expect \textsc{Patchy-san} to perform even better for large data sets. Moreover, \textsc{Patchy-san} is between $2$ and $8$ times more efficient than the most efficient graph kernel (WL). We expect the performance advantage to be much more pronounced for data sets with a large number of graphs. Results for betweeness centrality normalization are similar with the exception of the runtime which increases by about $10$\%. Logistic regression applied to \textsc{Patchy-san}'s receptive fields performs worse, indicating that \textsc{Patchy-san} works especially well in conjunction with CNNs which learn non-linear feature combinations and which share weights across receptive fields. 

\textsc{Patchy-san} is also highly competitive on the social graph data. It significantly outperforms the other two kernels  on four of the six data sets and achieves ties on the rest. Table~\ref{table-social-graphs} lists the results of the experiments.

\section{Conclusion and Future Work}

We proposed a framework for learning graph representations that are especially beneficial in conjunction with CNNs. It combines two complementary procedures: (a) selecting a sequence of nodes that covers large parts of the graph and (b) generating local normalized neighborhood representations for each of the nodes in the sequence.
Experiments show that the approach is competitive with state of the art graph kernels. 

Directions for future work include the use of alternative neural network architectures such as RNNs; combining different receptive field sizes; pretraining with RBMs and autoencoders; and statistical relational models based on the ideas of the approach. 
\vspace{-1mm}
\section*{Acknowledgments}
Many thanks to the anonymous ICML reviewers who provided tremendously helpful comments. The research leading to these results has received
funding from the European Union's Horizon 2020 innovation action program under grant agreement No 653449-TYPES.


\bibliography{dlnf}
\bibliographystyle{icml2016}

\end{document}